\documentclass[journal]{IEEEtran}
\IEEEoverridecommandlockouts
\usepackage{cite}
\usepackage{amsmath,amssymb,amsfonts}
\usepackage{graphicx}
\usepackage{url}
\usepackage{epsfig,epstopdf}
\usepackage{algorithm}
\usepackage{algorithmic}
\usepackage{booktabs}
\usepackage{float}
\usepackage{tabularx}
\usepackage{multicol,multirow}
\usepackage{textcomp}
\usepackage{xcolor}
\usepackage{bm}
\usepackage{amsthm}
\newtheorem{definition}{Definition}

\newtheorem{theorem}{Theorem}
\usepackage{array}

\usepackage[caption=false,font=small,labelfont=sf,textfont=sf]{subfig}
\begin{document}
\bibliographystyle{IEEEtran}
\makeatother

\title{Long-Term Client Selection for Federated Learning with Non-IID Data: A Truthful Auction Approach}
\author{Jinghong Tan, ~Zhian Liu, ~Kun Guo,~\IEEEmembership{Member,~IEEE}, ~Mingxiong Zhao,~\IEEEmembership{Member,~IEEE}
\thanks{This work was supported in part by the National Natural Science Foundation of China under Grants 61801418, 62301222 and 62361056, in part by the open research fund of National Mobile Communications Research Laboratory, Southeast University, under Grant No. 2023D04, in part by Applied Basic Research Foundation of Yunnan Province under Grants 202301AT070198, 202201AT070203 and 202301AT070422, in part by Yunnan Provincial Department of Education Science Research Foundation 2023J0025, in part by the Opening Foundation of Yunnan Key Laboratory of Smart City in Cyberspace Security (No. 202105AG070010-ZN-10), in part by the Innovation Foundation of Engineering Research Center of Integration and Application of Digital Learning Technology, Ministry of Education, China (No. 1431007), and in part by the Foundation of Yunnan Key Laboratory of Service Computing (No. YNSC24106).} 
\thanks{Jinghong Tan, Zhian Liu, and Mingxiong~Zhao are with the National Pilot School of Software, Yunnan University, Kunming 650500, China, with the Engineering Research Center of Integration and Application of Digital Learning Technology, Ministry of Education, Beijing 100039, China, with the Engineering Research Center of Cyberspace, Ministry of Education, Kunming 650504, China, and also with Yunnan Key Laboratory of Service Computing, Yunnan University of Finance and Economics, Kunming 650221, China (E-mails: phoilsa\_liu@mail.ynu.edu.cn, jinghong\_tansutd@hotmail.com, jimmyzmx@gmail.com, mx\_zhao@ynu.edu.cn).}
\thanks{
K. Guo is with the School of Communications and Electronics Engineering, East China Normal University, Shanghai 200241, China, and also with the National Mobile Communications Research Laboratory, Southeast University, Nanjing 210096, China. (email: kguo@cee.ecnu.edu.cn)}
\thanks{\emph{Corresponding author: Mingxiong Zhao (e-mail:~jimmyzmx@gmail.com)}.}
}

\maketitle
\begin{abstract}
Federated learning (FL) provides a decentralized framework that enables universal model training through collaborative efforts on mobile nodes, such as smart vehicles in the Internet of Vehicles (IoV). Each smart vehicle acts as a mobile client, contributing to the process without uploading local data. This method leverages non-independent and identically distributed (non-IID) training data from different vehicles, influenced by various driving patterns and environmental conditions, which can significantly impact model convergence and accuracy. Although client selection can be a feasible solution for non-IID issues, it faces challenges related to selection metrics. Traditional metrics evaluate client data quality independently per round and require client selection after all clients complete local training, leading to resource wastage from unused training results. In the IoV context, where vehicles have limited connectivity and computational resources, information asymmetry in client selection risks clients submitting false information, potentially making the selection ineffective. To tackle these challenges, we propose a novel Long-term Client-Selection Federated Learning based on Truthful Auction (LCSFLA). This scheme maximizes social welfare with consideration of long-term data quality using a new assessment mechanism and energy costs, and the advised auction mechanism with a deposit requirement incentivizes client participation and ensures information truthfulness. We theoretically prove the incentive compatibility and individual rationality of the advised incentive mechanism. Experimental results on various datasets, including those from IoV scenarios, demonstrate its effectiveness in mitigating performance degradation caused by non-IID data.
\end{abstract}
\vspace{-0.25em}
\section{Introduction}
The progress in computing and communication capacities \cite{r1}, along with the widespread adoption of smart devices, has facilitated the integration of machine learning into various mobile applications \cite{r2}. Privacy concerns have driven the rise of Federated Learning (FL) \cite{zhu2019multi}, a decentralized approach allowing mobile clients (MCs) to send local model updates to a central server (CS) without sharing raw data.
FL is widely used in the Internet of Things, Internet of Vehicles (IoV), and mobile edge computing \cite{wang2019edge,8771220,10102331,sun2024}.
In the IoV, smart vehicles generate vast amounts of data that can enhance machine learning models for safety, traffic management, and autonomous driving \cite{alalwany2024security}. 
Each smart vehicle participates in FL as the mobile client in the IoV. 
The application of FL in the IoV ensures that sensitive information, such as precise locations and driving behaviors \cite{korba2024}, is not exposed during the training process.


Nevertheless, the distribution of training data depends on individual usage patterns. 
 {For example, data from different vehicles varies due to differences in driving habits, locations, and environmental conditions, leading to the presence of data heterogeneity that is not independent and identically distributed (non-IID). This means local data may not represent the global data distribution, which can introduce bias and reduce model accuracy \cite{mcmahan2017communication}. For non-IID data, the gap between local and global model parameters is larger than in IID data, and this divergence builds up over communication rounds \cite{zhao2018federated}.}
\subsection{Motivation}
 {Prior research has proposed data sharing and data augmentation methods to address this issue from a data-oriented perspective \cite{zhu2021federated}.} However, sharing partial data poses a significant challenge in terms of privacy protection, which is a critical concern in the context of  {FL}, especially within IoV scenarios where the stakes of data breaches can be high.  {Smart vehicles, being constantly on the move and interacting with dynamic environments, generate highly sensitive and varied data that require robust privacy-preserving mechanisms \cite{alalwany2024security}. This motivates the exploration of alternative approaches that can address data heterogeneity while safeguarding privacy in IoV and other FL applications.
Recent efforts have tried to address the non-IID issue by clustering MCs based on their data distribution. This cluster information is used to select MCs for the next training round, creating a dataset with IID characteristics, without requiring local data to be uploaded \cite{lu2023auction, tian2022wscc}. However, in practice, data distributions can change, making implementation more complex. When data distributions differ significantly, clustering becomes difficult due to the lack of common traits, leading to the failure of clustering schemes in handling non-IID issues.
Metric-based client selection methods can address these limitations, but they face two challenges: }1) designing appropriate metrics, and 2) ensuring truthfulness in the reported metrics.

To address the first challenge, it is critical to design a metric that not only captures the long-term influence of client selection schemes but also is readily achievable within a reasonable timeframe. However, metrics proposed in previous works are evaluated independently round by round \cite{deng2021auction, zhang2021client, zhang2023dpp}, which provides only a one-shot optimal client selection but disregards the long-term accumulated influence of the selected clients. This flaw means that these approaches lie in their necessity to address the non-IID problem within a single communication round. 
Additionally, metrics proposed in other works \cite{cho2022towards, luping2019cmfl, cao2022birds, marnissi2024client} prevent the execution of client selection before local training begins. These metrics rely on information obtained from local training results, such as local accuracy or local model performance. As a result, this reliance on local training results leads to an ineffective utilization of computational resources. 
This negative influence is further exacerbated when local clients possess either a substantial volume of data or encounter poor communication channels, resulting in extended delays in acquiring local training results by the CS.

The second challenge lies in the client selection scheme's effectiveness in solving the non-IID issues, which heavily depends on the information collected from the MCs \cite{huang2022stochastic}. There exists an information asymmetry problem between the CS and MCs in realistic scenarios like IoV \cite{alalwany2024security}, where the CS may not have access to the true information of the MCs.
 {In resource-constrained wireless networks, if a poor channel MC declares much higher data quality than the real one, the CS might allocate superfluous bandwidth to this MC to benefit from its declared high data quality, resulting in insufficient bandwidth for truthful bidding MCs\cite{10570525}.
In such cases, the selection scheme becomes meaningless for the entire FL system, and may even be harmful \cite{liu2021privacy}.}
Common ways to incentivize MCs to submit truthful bids include game theory, contract theory, and auction. Using game theory for incentives, such as the Stackelberg game \cite{khan2020federated}, maximizes the benefits for one party rather than benefiting everyone.  {Additionally, contract theory \cite{li2022contract} is challenging to solve the problem of incomplete contract\cite{hart1988incomplete}.} Therefore, the auction is adopted to realize the incentives for MCs to behave honestly. 
\subsection{Our Approach and Contributions}
To address the non-IID problem, we try to capture the long-term impact of client selection and achieve a faster and more timely selection scheme.  {At the same time, ensuring the truthfulness of the metric under information asymmetry is key to achieving the effectiveness of the client selection scheme to address the non-IID problem.}

For this purpose, we propose a scheme called Long-Term Client-Selection Federated Learning based on Truthful Auction (LCSFLA). 
 {This approach focuses on long-term data balance to speed up FL convergence. The goal is to balance the amount of training data across different categories using the local data distribution of MCs in multiple communication rounds, without needing local training results.
Guided by this principle and to assess the long-term influence of the selected MCs on the global model, we introduce a metric, called ``Data Category Discrepancy" (DCD), which evaluates the difference in data size between each data category and a specified reference category. 
The reference category is the highest cumulative training data volume obtained up to the current training iteration.
Based on this, reducing the DCD across categories helps achieve data balance.
Therefore, to further evaluate the individual MCs’ contribution to reducing the DCD across categories, i.e., its data quality, we propose a novel evaluation metric, called ``Unit data quality" (UDQ), based on the DCD and local data distribution.
This parameter can identify MCs that effectively achieve data balance, ultimately helping us to address non-IID issues.}

 {Furthermore, the truthfulness of the evaluated contribution is critical in achieving data balance because if the information submitted by the MC is false, the client selection scheme won't be effective.}
Thus, the calculation of  {UDQ} must be trusted. 
To this end, we have devised an incentive mechanism scheme based on the Vickrey-Clarke-Groves (VCG) auction \cite{rVCG}. 
 {This mechanism aligns the interests of individual MCs with the overall FL system and encourages MCs to submit honest information necessary for evaluating truthful contributions. 
Besides the truthfulness guarantee, we introduce a deposit to avoid premature MC opt-out before the CS gets their local training results. Moreover, a theoretical analysis of incentive compatibility (IC) and individual rationality (IR) was provided to ensure the scheme's robustness.}
Based on this, we propose a social welfare maximization problem that aims to select the MCs effectively data balance while balancing the sum of data quality and actual energy costs of selected MCs. This is achieved through the joint optimization of bandwidth allocation, local iterations, and client selection for FL in the wireless network.

The contributions of this paper are summarized as follows
\begin{itemize}
    \item 
    We introduce a novel long-term client selection scheme that tackles the challenge of non-IID data in FL. This scheme leverages designed metrics to address the cumulative negative effects from a fresh perspective of data balance, ultimately accelerating FL model convergence.
     {The core of this scheme lies in a novel metric UDQ, which assesses the long-term contribution of MCs in achieving data balance.}
    \item 
     {To select the appropriate clients for the current round, along with the number of local iterations and bandwidth allocation, an optimization problem was formulated.} This optimization problem considers the constraints of communication resources and aims to achieve a balance between MC contributions and energy costs by maximizing social welfare, thereby addressing the non-IID issue with minimal energy costs.
    \item 
     {To ensure that MCs upload truthful information, allowing for accurate data quality assessment and optimization of social welfare, an incentive mechanism based on the VCG auction was designed, in which MCs submit an upfront deposit to prevent opt-out while the CS commits to a reward at the end of the communication round.} This design ensures that both IC and IR conditions are met.
    \item 
    Finally, our extensive simulations show that LCSFLA significantly speeds up model convergence. Compared to the baselines on different datasets, LCSFLA achieves: 1) Up to 2\%-61\% higher accuracy 2) final target accuracy with only 20\%-75\% of the communication rounds, and 3) target accuracy with only 32\%-87\% of the energy cost. 
\end{itemize}

The rest of this paper is organized as follows. Section II introduces related work in this field. Section III makes an illustration of our system model. In Section IV, we present the design of the auction mechanism. Section V shows our numerical simulation results, and Section VI concludes.

\emph{Notations:} Scalars, column vectors, matrices, and sets are denoted by unbold letters, lower-case bold letters, uppercase bold letters, and calligraphy letters, respectively, e.g., $a$, $\bm{a}$, $\bm{A}$ and $\mathcal{A}$. The notation $\bm{A}^T$ denotes the transpose of matrix $\bm{A}$, while $\bm{a} \in \mathbf{R}^{n \times 1}$ signifies a collection of $n$-dimensional real vectors. The expectation of random variable $x$ is represented by $\mathbb{E}[x]$.

\section{Related work}
Our research focuses on two key areas: non-IID issues and incentive mechanisms. We'll delve into related work in these areas next.
\subsection{Non-IID Issue}
The non-IID problem has been studied since the inception of FL. For example, the paper \cite{mcmahan2017communication} shows that FedAvg suffers a 37\% loss of accuracy on the CIFAR-10 dataset in non-IID scenarios. In the paper \cite{zhao2018federated}, simulation results even suggest that the FedAvg algorithm becomes very sensitive to the distribution of mobile user data and may fail to converge on strong non-IID data, especially when using deep neural networks. To address this issue, the authors propose the concept of data sharing in\cite{zhao2018federated}. Experimental results show that the test accuracy of the model can be improved by about 30\% on the CIFAR10 dataset with only 5\% globally shared data from each MC. However, downloading a portion of the shared dataset to each mobile user for model training violates the requirement of privacy-preserving learning, which is the fundamental motivation of FL.
To avoid privacy leakage, some works have grouped MCs with similar data distributions \cite{briggs2020federated}. Then, the models were trained in carefully selected groups\cite{tian2022wscc}. However, such clustering strategies may fail in scenarios with complex data distributions and need to cluster again when the data distribution of the local dataset varies.
To solve these drawbacks, the strategies of metric-based client selection become a greater solution for non-IID problems in  {FL}. 
Such as a dynamic evaluation model is proposed \cite{deng2021auction}. Through each round of independent calculation based on indicators like data size, data distribution, and error labeling ratio, the CS can obtain the data quality of each MC, and select the MC with higher data quality. However, this scheme is only selected independently in each communication round, without considering the long-term influence of the selected MC.
To capture the long-term influence, in some works \cite{zhang2021client,cho2022towards, Dynamic2021guo}, the CS calculated the parameter gap between the local model and the global model in each communication round. Then the CS selected the MCs with a lower parameter gap \cite{zhang2021client} or chosen MCs with significant parameter gap to participate in aggregation \cite{cho2022towards}. A smaller parameter gap indicates that the local model is closer to the global model. 
To simplify the calculation of similarity, some work using a selection strategy based on the importance of the gradient norms\cite{marnissi2024client} or the local loss of the MC was utilized as a replacement for the gradient norm \cite{cao2022birds}.

However, these schemes implement client selection after all candidate MCs obtain their local train result. This is because metrics calculation relies on local training outcomes for MCs, such as local accuracy or local model parameters. Therefore, these schemes\cite{zhang2021client,cho2022towards, Dynamic2021guo, marnissi2024client} cause massive unused local training, leading to the reduction in utilization efficiency of local computation resources.

Thus, we use the perspective of data balance to guide client selection to capture the long-term influence of the selected MC meanwhile no need for the result of local training. 

\subsection{Incentive Mechanism Design}
In the model training of  {FL}, the MCs, typically consume their resources of computing and communication resources, for local training. This prevents self-interested MCs from contributing their resources to FL unless provide correlated rewards\cite{tu2022incentive}.
Moreover, with information asymmetry between the CS and MCs, MCs will want to get inordinate rewards by submitting no-real information, which has a great negative influence on the performance of the model \cite{liu2021privacy}.
To avoid the above issue, it is necessary to design an incentive mechanism to encourage honest behavior. Moreover, suppose it is non-trivial to recover the mapping between the information submitted by MCs and a metric that measures the contribution of MCs, the motivation to submit false information is greatly reduced.

Shapley Value (SV) has been explored in FL incentives to assess the data quality of MCs to their local datasets\cite{wang2020principled}. While SV offers a fair way to allocate rewards based on marginal utility, it's computationally expensive. Calculating MC contributions requires running many global training rounds, consuming significant resources. 
Therefore, some approaches \cite{kang2019incentive} rely on local training results to assess contribution with much less computation. However, MCs seeking higher rewards from the CS can easily manipulate local accuracy.
Research suggests \cite{lu2023auction,deng2021auction} using local data distribution as an evaluation criterion to address the issue of local accuracy manipulation. For example, the local data sample concentration is factored into service cost calculations for achieving Nash equilibrium \cite{lu2023auction}, or the data labels and full data distribution are utilized as the input of reinforcement learning for client selection \cite{deng2021auction}. 
However, if CS only uses the data distribution as the evaluation criterion, MCs can still infer the evaluation principle of CS through multiple interactions with CS. MCs will know that larger local data scales and more singular data concentrations can be considered as higher data quality and can generate more profits.
Thus, by manipulating their data distribution, MCs could cheat the CS to obtain higher rewards. 
To address the above issues, we use the unit data quality, which is based on the DCD and data distribution, as the basis of assessment. Since MCs know neither others' distribution nor the MCs selected in the previous communication round, they cannot predict the DCD in the current communication round. Therefore, it is difficult to gain more profit from manipulating existing information. This removes the motivation for manipulation and ensures the truthfulness of the information reported to the CS.

In addition, various incentive mechanisms were another perspective to encourage MCs to upload real information.
Early works are partial to maximizing the utility of a single party \cite{weng2019deepchain,khan2020federated,kang2019incentive,8422684}. 
These approaches employed reputation-based mechanisms \cite{weng2019deepchain} for encouraging honest participation and Stackelberg game-based incentives \cite{khan2020federated} to maximize the utility of leader (CS). 
Some studies have used contract theory \cite{nisan2007algorithmic} to maximize the utility of other roles \cite{kang2019incentive, 8422684}.
For example, the local training accuracy is used as the basis for contract design \cite{kang2019incentive}, and authors have established the incentive mechanism maximizing the utility of the data holder based on the contract theory. Another study \cite{8422684}, using contract theory, formulated a negotiation process between task publishers and fog nodes that optimizes task publisher utility.
However, focusing solely on maximizing one party's benefit can lead to inequities, discouraging participation from others.
Recognizing these limitations, recent research has shifted towards total social welfare optimization to encourage the participation of both parties. This can be realized by Auction-based solutions \cite{he2019truthful,nisan2007algorithmic,wu2022sustainable}.
Authors in \cite{he2019truthful} use a primitive-double greedy auction mechanism and apply Myerson theorem \cite{nisan2007algorithmic} to ensure incentive compatibility and individual rationality of the mechanism.
However, computing the optimal virtual reward function under Myerson mechanisms is NP-complete \cite{nisan2007algorithmic}.
To address it, some works use the VCG mechanism \cite{rVCG} as an auction mechanism aimed at maximizing total utility \cite{wu2022sustainable}. 
The payment of each MC can be calculated by CS in polynomial time in the VCG mechanism.

Thus, we designed an incentive mechanism that utilizes a payment flow inspired by the VCG auction payment scheme. This design ensures the uniformity of individual and collective benefits. In other words, for each MC, uploading truthful information becomes the best course of action, as it avoids harming its own interests while maximizing social welfare.


\section{System Model}\label{System Model}
\subsection{FL Execution Flow}
\begin{figure*}[ht]
    \centering
    \includegraphics[width=0.9\textwidth]{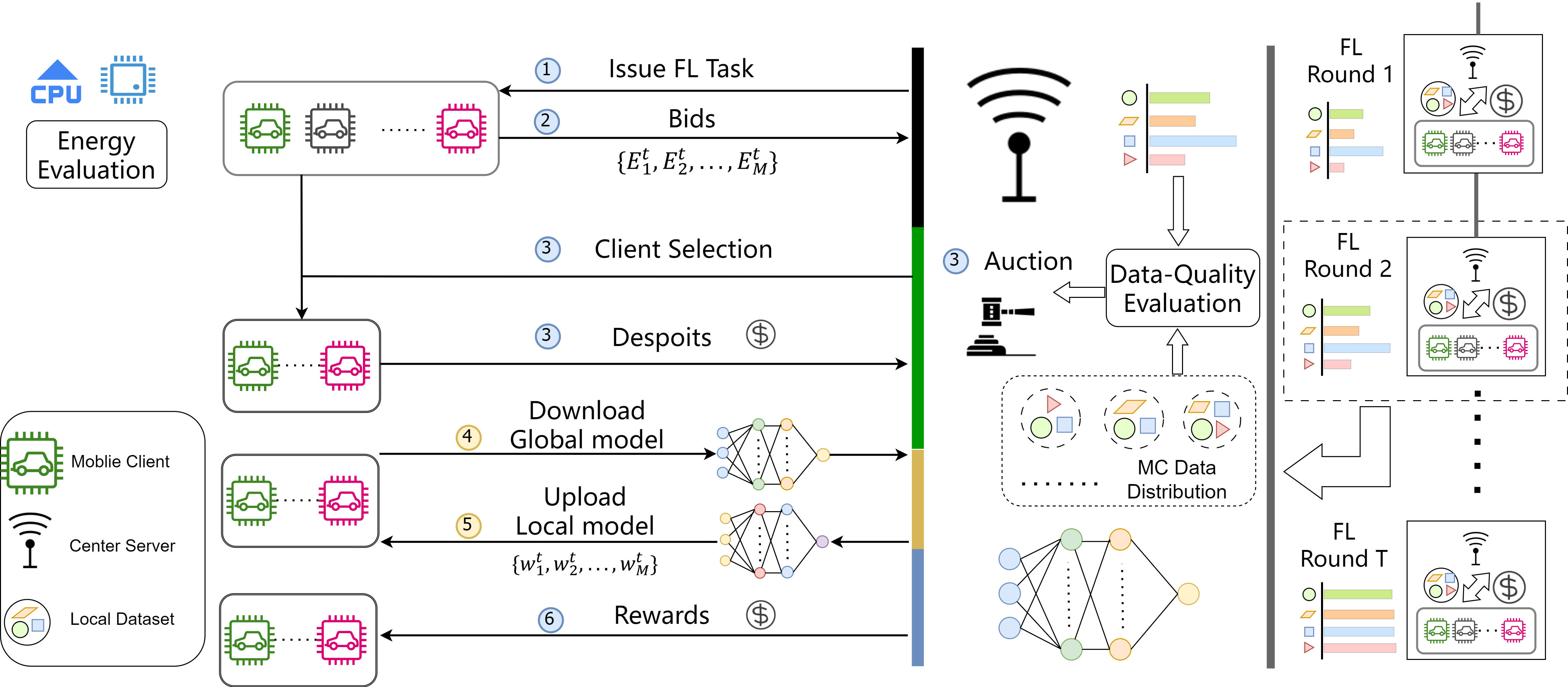}
    \caption{The figure illustrates the entire workflow of LCSFLA,  from task publication to MC receipt of rewards. Firstly, the explanation for the third step shown in the figure, which is also our main focus, is provided. In this step, the CS combines the current data size after training, represented by the bar chart above the evaluation module, with the local data distribution of MCs, represented by the dashed circle within the geometric figure below the evaluation module, to conduct a data quality evaluation for each MC. This result of assessment becomes an essential consideration in the client selection process.}
    \label{Fl-flow}
\end{figure*}
This paper considers a typical FL service market in the IoV scenarios, where a central server coordinates with smart vehicles interested in participating in FL, i.e., mobile clients represented by $\mathcal{M} = \{1, 2, \cdots, M\}$. Each MC possesses a local dataset $\mathcal{D}_m = \{\bm{X}_m, \bm{Y}_m\}, \forall m\in\mathcal{M}$ with a size of $D_m$, where $\bm{X}_m = [\bm{x}_{m, 1}, \cdots, \ \bm{x}_{m, D_m}]^T, \forall m\in \mathcal{M}$ and $\bm{Y}_m = [\bm{y}_{m, 1},\ \cdots, \ \bm{y}_{m, D_m}]^T, \forall m\in \mathcal{M}$ represent the data points and their corresponding labels, respectively. 
The CS leverages MCs to train a shared global model by aggregating their local models. This aggregation occurs in a distributed manner. The CS assigns learning tasks to MCs and incentivizes their participation with rewards. Selected MCs contribute computing power and local data to complete the assigned tasks. The system's workflow can be summarized in several interconnected steps, as illustrated in Fig. \ref{Fl-flow}.









In each communication round, the FL platform issues FL tasks, and the CS invites MCs to participate in the task training. 
After accepting the invitation, the candidate MCs submit bids regarding their computation consumption for one local iteration, estimated channel coefficients, and prescribed transmission rates for all MCs.
Additionally, suppose an MC is participating in a learning task for the first time. In that case, they need to submit the local data distribution to the CS for data quality evaluation \footnote{If data distribution evolves with time, our proposed framework still works by requiring MC to submit a data distribution vector in each communication round. For example, assuming our experimental environment is the MNIST dataset, the MC will submit a vector with the size of each data category, as done in \cite{deng2021auction}.} 
Based on the information submitted by the MCs, the CS evaluates the data quality of each MC according to the  {UDQ} described in the next subsection. 
Then, the CS goes through the auction market to select a set of winning MCs while collecting the appropriate deposit from the selected MCs.
If MC $m$ wins FL task in communication round $t$, ${q}_{m}^t=1$. Otherwise, ${q}_{m}^t=0$.  {In other words, each MC can only be selected once in communication round $t$.}
Besides, the CS can select up to $N$ MC, i.e.,
\begin{equation}
    \label{task restrain 2}
    \sum_{m=1}^{M}{q}_{m}^t = N.
\end{equation}
Subsequently, the CS distributed the current global model $\bm{\omega}^t$ to the selected MCs.
These MCs train their local model starting with the current global model on its dataset $\mathcal{D}_m$. 
This process can be represented as follows
\begin{equation}
    \label{local_FL}
    F_m(\bm{\omega}^t)=\frac{1}{D_m}\sum_{i=1}^{D_m}{f(\bm{\omega}^t, \bm{X}_{m}, \bm{Y}_{m, i})},
\end{equation}
where $f(\bm{\omega}^t, {x}_{m, i}, {y}_{m, i})$ is the loss of model  $\bm{\omega}^t$ in the data sample $(x_{m, i},\ y_{m, i})$. The gradient is then computed sequentially on each batch, and the local model of MC $m$ is updated toward minimizing the loss function, in which the stochastic gradient descent method was used to update the model with local loss as 
\begin{equation}
    \label{local-updata}
    \bm{\omega}^{t}_{m, l}=\bm{\omega}^{t}_{m, l-1}-\eta\frac{\partial F_m(\bm{\omega}^t_{m, l-1})}{\partial\bm{\omega}^{t}_{m,l-1}},l=1\dots,l_m^t,
\end{equation}
where $\bm{\omega}^t_{m, l}$ is the updated local model of MC $m$ in the local iterations $l$, the parameter $\eta$ is the local learning rate and $l_m^t$ is the the number of local iterations for MC $m$. In the next step, each MC submits its updated local model parameter $\omega^t_m$ to CS for model aggregation. 
According to \cite{wang2020tackling}, the larger the local iteration is, the higher the local accuracy is, and the more important it will be to train the global model.
Thus, we introduced the consideration of client-specific local iteration numbers during model weight aggregation to improve model performance, namely 
\begin{equation}
    \label{glabel_updata}
    \bm{\omega}^{t+1}=\sum_{m=1}^{M}{\frac{\tau_m^t}{\tau^t}\bm{\omega}^t_m},
\end{equation}
where $\tau^t=\sum_{m=1}^{M}\tau_m^t$, and $\tau_m^t = {{q}_m^{t}}^T {l}_m^t D_{m}$ are the aggregation weights, with ${q}_m^t$ as the winner indicator parameter for MC $m$ at the communication round $t$ and ${l}_m^t$ as the numbers of local iterations MC $m$. The number of selected MCs during a communication round is denoted as $N$. 
Finally, the CS distributes rewards to the MCs based on the number of local iterations they perform. The entire FL process operates as a cohesive and interconnected workflow to facilitate collaborative model training while maintaining privacy across distributed devices.

\subsection{Evaluation of MC Data Quality}
Throughout the training process, in order to tackle the issues caused by statistical data heterogeneity among  MCs, the CS needs to carefully select a set of MCs based on their local data distribution and the current training status. At the same time, the long-term accumulated influence of the selected MCs should be considered. To achieve this end, we design a novel metric, called unit data quality (UQD), which metric evaluates the long-term data quality of each MC by considering the data statistics of the MCs and the DCD up to the current communication round.
This makes it possible to perform client selection before local training begins, effectively utilizing computational resources.
\subsubsection{Data Category Discrepancy}

Due to the MCs' local data heterogeneity, the cumulative discrepancy in the volume of trained data across distinct data categories will gradually vary during training.
To capture this variation resulting from the long-term accumulated influence of the selected MCs, we need to quantify the difference in the data amount learned during previous communication rounds between the dominant category \footnote{In this context, the term ``the dominant category" denotes that the data amount of this category is the largest up to the current communication round. Conversely, the `scarcest category' refers to the category with minimal data amount. \label{cate_diff}} and the rest of the categories.
To this end, we first need to calculate the absolute value of the cumulative training data volume. Based on this, we then compute the relative value reflecting the discrepancy across various categories.
Hence, we firstly need to define MC $m$'s data distribution vector $\bm{d}_m=[d_{1, m}, \ d_{2, m}, \dots, \ d_{Z, m}]^T \in \mathbb{Z}^{Z\times 1}$, where $Z$ is the amount of data category and $d_{z, m}$ is the data size of category $z$ of MC $m$. 
Based on the above information, the data amount of a specific category learned up to communication round $t$ can be expressed as the sum of data within the particular category of all participating MCs across from communication round $1$ to $t$.
The aggregated vector of the amount of data learned for all categories is denoted as $\bm{g}^t=[g_{1}^t, \ g_{2}^t, \ \dots, \ g_{Z}^t]^T \in \mathbb{Z}^{Z\times 1}$, where 
\begin{equation}
\label{cate_size}
g_{z}^{t}=\sum_{m\in \mathcal{M}}{{{q}_m^{t}} {l}_m^td_{z,m}}+g_z^{t-1},
\end{equation}
where $g_z^{t-1}$ represents the amount of learned data for the category $z$ up to the communication round $t-1$. 
Based on the absolute value calculated from \eqref{cate_size}, next, we need to calculate the relative value of cumulative training data volume across various categories, i.e., the data category difference between the amount of data for each category and the dominant category.
The largest $g_z^t$ is chosen as the dominant category and denoted as $g^t$, which is used to calculate the DCD vector $\bm{o}^t = [o_1^t, \ o_2^t, \ \dots, \ o_Z^t]^T\in\mathbb{Z}^{Z\times 1}$, where
\begin{equation}
\label{diff}
{o}_z^t = g^t-g_z^t.
\end{equation}

\subsubsection{Data Quality of MC}
\begin{figure*}
\centering
    \subfloat[]{\includegraphics[width=0.23\textwidth]{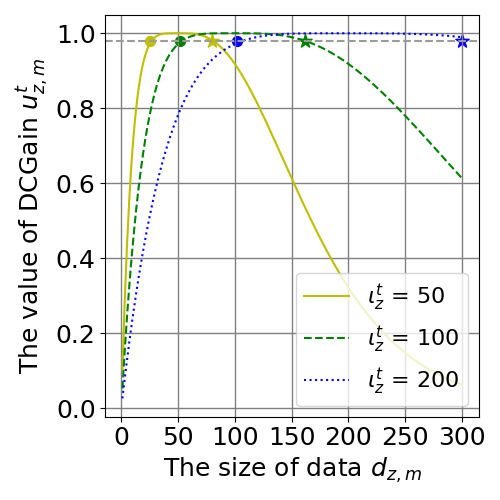}%
    \label{DCGain-1}}
\hfil
    \subfloat[]{\includegraphics[width=0.23\textwidth]{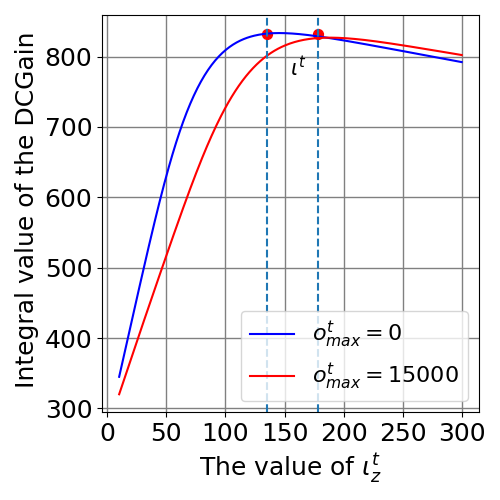}%
    \label{DCGain-3}}
\hfil
    \subfloat[]{\includegraphics[width=0.23\textwidth]{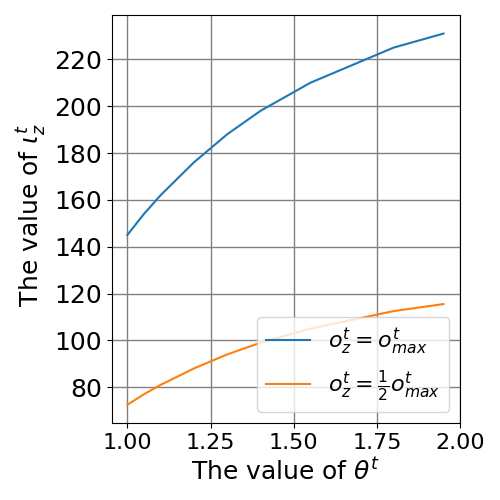}%
\label{DCGain-4}}
\caption{The sub-fig (a) shows UDQ curves to four different experiment setups, and it can be shown that the distance of the highest point in the data preference range from the origin becomes greater in the x-axis as $\iota_{z}^t$ increases. In addition, in sub-fig (a), we can know that the width of the data preference range will be broader as $\iota_z^t$ increases. Dots and asterisks indicate the starting point and endpoint of the data preference range, respectively. 
The sub-fig (b) displays the changing trend of the integrated value ${E}(u_{z, m}^t)$ as $\iota_{ z}^t$ increases, and the value of $\iota_{z}^t$ from which we can obtain the maximum of the integrated value is defined as $\iota^t$. Lastly, in sub-fig (c), we observe that the $\iota^t$ increases when $\theta^t$ increases for the scarcest category. Furthermore, the value of $\iota_z^t$ also increases as $\theta^t$ increases when ${o_z^t} = \frac{1}{2}{o_\text{max}^t}$.
In conclusion, the value of $\iota_z^t$ will increase as $\theta^t$ increases when the relative scarcity of any category is constant, i.e., ${o_z^t}/{o_\text{max}^t}$ remains constant.
It is important to note that $\theta^t$ increases solely because $o_\text{max}^t$ increases.
The parameters in the three sub-figures are set as follows: $o_\text{max}^t =1000$, $ d_\text{avg} =200$, $\alpha=1$. Other parameters are set as described in Section V-A.}
    \label{fig: DCGain}
\end{figure*}

It is known that the smaller the DCDs for all categories are, the better balanced the accumulated trained data sets across various categories are. 
The less trained a specific category is, the more important it becomes to replenish data in that category in the next round. The degree of data deficiency in a category determines the extent of data supplementation required. In a simple scenario, where the MC possesses a single category of data and a fixed size of local sets, selecting $N$ MCs with the largest data amount of the scarcest category\textsuperscript{\ref{cate_diff}} would achieve a balanced data distribution for that category. Subsequently, in the following communication rounds, the MCs with the largest data amount of the sub-scarce category are selected to achieve a data balance of all categories gradually. However, in the realistic scenario, where both the local data categories and size of MCs are non-IID, the complex components of local datasets of the selected MCs make it difficult to supplement the DCD gap of the current communication round. Thus, client selection is a complex issue in design. 

To address it, we design a data quality metric $c_m^t$\footnote{During training, the distribution of the trained dataset will change, hence $c_m^t$ and $u_{z, m}^t$ vary across communication rounds.\label{U_C}} to quantify the contribution of MC $m$ to assist us in client selection. 
We first consider a simple case, i.e., the MCs are allowed to have only a single data category.
When a single-category MC possesses a specific data size in category $z$ that can perfectly fill the DCD gap of the category $z$ from a long-term standpoint, that MC is the ideal candidate to be selected for category $z$. Therefore, When comparing to the single-category MCs with data size that is less or more than the specific data size in category $z$, these ideal MCs will be considered to have the highest importance. 
On the other hand, the scarcity of different categories differs and the scarcity of category $z$ is reflected by $o_z^t$.
The larger DCD $o_z^t$ is, the scarcer the data category $z$ is. 
Therefore, among the ideal single-category MCs, although they can perfectly fill the DCD gaps, their contributions vary due to the different scarcity levels of the categories.
The MCs with data of the scarcest categories will be considered the most significant.
Based on the above analyses, the data quality of MC $m$ will be formulated as the product of the data size $d_{z,m}$ and the UDQ $u_{z, m}^t$, where the $u_{z, m}^t$\textsuperscript{\ref{U_C}} was used to denote the significance of each data sample, i.e., the UDQ. For the ideal single-category MCs, the data quality of a single data sample is set to the maximum value $\alpha$. 
The data quality of an individual data sample decreases and is less than the maximum value $\alpha$ if the data size owned by MCs is either less or more than the specific data size. 
For a category $z$ with a larger DCD $o_z^t$, the specific data size is greater, leading to a higher corresponding product value.
Furthermore, when the data sizes and categories of two single-category MCs are the same, the MC that performs a higher number of local iterations will lead to higher data quality since more local iterations typically lead to higher local accuracy for each MC. Thus, the data quality of MC $m$ is reformulated as the product of the UDQ $u_{z, m}^t$, the data size $d_{z,m}$, and local iterations $l_m^t$. 
However, the above formulation of data quality is not favorable to MCs with small data sizes, and these MCs might have many unique data samples \cite{balakrishnan2022diverse}.
Thus, when the data sizes, data categories, and local iterations of two single-category MCs are the same, the one with fewer historical training rounds will lead to higher data quality. The historical training rounds ${v_m^t}$ records the sum of communication rounds in which MC $m$ participates in FL training from communication round $1$ to $t$. 
The introduction of the historical training rounds can highlight MCs that are seldom to never trained to facilitate the incorporation of new features into the global model.
Overall, for the single-category MCs, the data quality $c_{z,m}^t$ of MC $m$ was formulated as follows
\begin{equation}
    \label{dataquality_1}
    c_{z,m}^t = \sigma  u_{z, m}^t d_{z,m} \lambda_m^t l_m^t,
\end{equation}
where $\sigma$ represents a normalized factor contributing to aligning the magnitude of energy consumption presented later in Section III, $l_m^t$ is the number of local iterations, and $\lambda_m^t=\beta^{v_m^t}$ acts as a  {``no-bias"} factor. Here, $\beta$ is a positive value less than 1. 
This factor ensures that the client selection scheme is unbiased with respect to the MCs' data volume, giving both MCs with large and small data volumes an opportunity to be selected.
When MC $m$ is selected in a communication round, ${v_m^t}$ increases by 1. 
Moreover, the UDQ $u_{z, m}^t$ in \eqref{dataquality_1}, as mentioned above, shows the preference of the client selection scheme to MC $m$ in category $z$, and its evaluation method will be described in detail later. In a realistic scenario, the MCs are allowed to possess multiple categories of data. Thus, for the multi-category MCs, the data quality $c_m^t$ is formulated as follows
\begin{equation}
    \label{dataquality_2}
    {c}_m^t= \sum\nolimits_{z=1}^{Z}{c_{z,m}^t}.
\end{equation}
Except for the UDQ $u_{z, m}^t$, the rest of the parameters are easy to obtain in \eqref{dataquality_1}. In the following, we will introduce the UDQ $u_{z, m}^t$ in detail. First, as previously described, for a particular DCD $o_z^t$, there is always a specific data size that can most effectively reduce the DCD of category $z$ from a long-term standpoint, and corresponding UDQ is maximal compared to other sizes in category $z$. When the $d_{z, m}$ is lesser or more than the specific data size, the corresponding UDQ decreases.
Therefore, plotting $u_{z, m}^t$ on the y-axis and $d_{z,m}$ on the x-axis yields the UDQ curve, which should exhibit an increasing and then decreasing trend. 

The highest point of this curve corresponds to the data amount owned by an MC that precisely compensates for the scarcity. 
In our work, we denote this highest point as the category reference parameter $\iota_z^t$, i.e., when $d_{z, m} =\iota_z^t$, the UDQ of MC $m$ is maximal in the category $z$.
After the category reference parameter $\iota_z^t$ is determined, the data size in the range near $\iota_z^t$ should be considered to be close to the significance of the specific data size, since this data size can also effectively reduce the DCD.
Even if there may not be a data size $d_{z, m}$ exactly equal to $\iota_z^t$ in a realistic scenario, the data sizes within this range can be seen as a perfect substitute. This data size range can be said to be the data preference range\footnote{In practice, if the value of the UDQ curve within a range are all greater than $\epsilon \alpha$, then this range can be considered as the data preference range. In our work, the $\epsilon$ was seen as 0.99.}, and the UDQ curve within this range is approximately a straight line parallel to the X-axis.
When the DCD $o_z^t$ is large, the client selection scheme requires a broad data preference range to prioritize the MCs with data sizes in category $z$. This approach tends to obtain sufficient data for category $z$ to fill the DCD gap for category $z$ effectively. 
On the other hand, for the smaller DCD $o_z^t$, the training data in this category is not as urgent as the categories with large DCD gaps. Nevertheless, the client selection scheme still needs a minimal amount of data from category $z$ to participate in FL, preventing the global model's gradient update direction from deviating from the ideal global gradient update direction in a single communication round.
Therefore, the width of the data preference range is allowed to be smaller, but not absent. 
To that end, we use the gain compensation parameter $\nu_z^t$ to control the slope of the UDQ curve.
Besides the $\iota_z^t$ and $\nu_z^t$, the data category diversification also influences the UDQ.
The MCs with more categories, i.e., greater data category diversification, contribute more significantly to the final performance of the global model \cite{deng2021auction} and thus deserve higher UDQ. 
Based on the above information, the UDQ $u_{z, m}^t$ in the category $z$ for MC $m$ at communication round $t$ is given as follows 
\begin{equation}
\label{DCGain}
u_{z, m}^t=\alpha\left[1 - (1-u^\text{c}_m) \left(\frac{\nu_z^t d_{z, m}-\iota_{ z}^t}{{\iota_{ z}^t}}\right)^2\right],
\end{equation}
where the parameter $\alpha$ is preset to control the high bound of  {UDQ}, and $u^{c}_m$ is the gain of data category diversification based on the local data distribution of MC $m$. The $\iota_{z}^t$ determines the data size at which the maximum UDQ is achieved for category $z$.
It follows that a larger $\iota_z^t$ leads to a greater UDQ for large data sizes. As shown in Fig. \ref{fig: DCGain} (a), for the large data size $d_{z, m}$ = 300, when $\iota_{z}^t$ = 100, $u_{z, m}^t$ is 0.65, and when $\iota_{z}^t$ = 200, $u_{z, m}^t$ is 0.99. Moreover, the value of $\nu_z^t$ determines the slope of the curve.
The specific descriptions of these different factors are described in detail below.

First, to determine the value of $\iota_z^t$, we need to look at the scarcest category, denoted as $z_\text{max}^t$. 
Due to the proposed client selection scheme requires that as much data as possible is selected from the scarcest category, the mathematical expectation of the UDQ of the scarcest category with respect to $d_{z, m}$ from 0 to maximal data size $d_\text{max}$ should be maximal compared to any other category.
For intuition, the mathematical expectation can be visualized by assuming the local data distribution is IID. Specifically, in the IID case, this mathematical expectation is the area enclosed under the UDQ curve, which differs for different $\iota_{z}^t$. 
The probability associated with each data size value differs in general cases where the local data distribution is non-IID. Therefore, it is necessary to incorporate the probability density function when computing the mathematical expectation, i.e.,
\begin{equation}
    \label{E}
    \mathbb{E}(u_{z, m}^t) = \int_{0}^{d_\text{max}} \Gamma(d_{z, m}) u_{z, m}^t d (d_{z, m}),\ \text{when}\ z = z_\text{max}^t,
\end{equation}
where $\Gamma(d_{z, m})$\footnote{Of course, it is not possible for us to obtain the actual probability density function of the real world. However, we can obtain the probability density function of participant MCs by the data distribution uploaded by the MCs.} is the probability density function of $ d_{z, m}$. 
At this point, we use \eqref{p_max} to determine the value of $\iota_{z}^t$ for $z_\text{max}^t$, denoted as $\iota^t$, which represents the optimal specific size corresponding to the scarcest category $z_\text{max}^t$ in the communication round $t$, as follows
\begin{equation}
    \label{p_max}
    \iota^t = \text{arg} \max_{\iota_z^{t}}\mathbb{E}(u_{z, m}^t),\ \text{when}\ z = z_\text{max}^t.
\end{equation}
After obtaining the category reference parameter $\iota^t$ for category $z_\text{max}^t$, we can calculate the category reference parameter $\iota_z^t$ for other categories. 
Since a smaller DCD $o_z^t$ indicates that the amount of data of category $z$ learned up to communication round $t$ is greater, there will be less demand on the client selection scheme for data in category $z$. This means the corresponding ideal specific data size $\iota_z^t$ should be smaller. 
So the category reference parameter $\iota_z^t$ can be calculated as follows
\begin{equation}
\label{pz}
\iota_z^t=\iota^t{o}_z^{t-1}/{{o}_\text{max}^{t-1}},
\end{equation}
where ${o}_\text{max}^{t-1}$ represents the DCD of scarcest category in communication round $t-1$. 
The position of $\iota_z^t$ on the X-axis will influence the position of the data preference range on the X-axis.
As Fig. \ref{fig: DCGain} (a) shows, the larger category reference parameter $\iota_{z}^t$ is, the bigger the distance from the highest point of the data preference range to the origin on the X-axis is. 

Next, the width of the data preference range is another parameter that needs to be considered besides the position of the data preference range. A broader data preference range can capture more data to fill the DCD gap.
To change the width of the data preference range, we need to adjust the value of $\nu_z^t$. For a scarcer category, the corresponding width of the data preference range should be broader. 
To achieve this, considering that \eqref{DCGain} is quadratic with respect to $d_{z, m}$, and since the UDQ curve reaches its maximum when $\nu_z^t d_{z, m} = \iota_z^t$ and the value of $u_{z, m}^t$ is larger when $\nu_z^t d_{z, m}$ is closer to $\iota_z^t$, the designed gain compensation parameter $\nu_z^t$ is a decreasing function of $d_{z, m}$, i.e.,

\begin{equation}
\label{mp}
\nu_z^t = \exp\left[1-\left(\frac{d_{z,m}}{\iota_{ z}^t}\right)\right],
\end{equation}
For a certain communication round $t$, the DCD $o_\text{max}^t$ is a fixed value and the value of $o_z^t$ differs for different categories. 
In addition, since $\iota_z^t$ increases when $o_z^t$ increases in a communication round, we need the width of the data preference range to broaden as $\iota_z^t$ increases. A concrete illustration is provided in Proposition 1.
 {
\begin{theorem}
\label{p1}
    The width of the data preference range increases monotonically with $\iota_{z}^t$. 
    Specifically, the presence of $\nu_z^t$ should ensure that the distance from the starting point to $d_{z,m} = \iota_{z}^t$ and the distance from the endpoint to $d_{z,m} = \iota_{z}^t$ both increase as $\iota_{z}^t$ increases. 
\end{theorem}
\begin{proof}
\label{p1-1}
The detailed proof is presented in Appendix A.
\end{proof}}

However, the corresponding $o_\text{max}^t$ is different for various communication rounds. To effectively capture more data to reduce the DCD of the scarcest category for a larger $o_{\text{max}}^t$ across communication rounds, the client selection scheme requires a broader data preference range compared to a small $o_{\text{max}}^t$ in the scarcest category. Even if the DCD $o_{\text{max}}^t$ is larger, the client selection scheme expects all MCs possessing data of category $z_\text{max}^t$ to be selected to participate in FL. As mentioned above, the \eqref{mp} need be introduced a parameter $\theta^t$ related to $o_\text{max}^t$ and was reformulated as follows
\begin{equation}
\label{mp01}
\nu_z^t = \exp\left[1-\left(\frac{d_{z,m}}{\iota_{ z}^t}\right)^{\theta^t}\right],
\end{equation}
where $\theta^t = {\log_{\vartheta}{(\vartheta+o_\text{max}^t/{Nd_\text{avg}})}}$, with $\vartheta$ as the super-parameter greater than zero and $d_\text{avg}$ as the average value of data size for each category, namely $d_\text{avg} ={\sum_{m\in \mathcal{M}}D_{m}}/{\sum_{m\in\mathcal{M}}z_m}$. The parameter $\theta^t$ aims at influencing $\nu_z^t$. It is known that if $\nu_z^t d_{z, m}$ is closer to $\iota_z^t$, the value of $u_{z, m}^t$ will increase when $d_{z, m}$ is a fixed value. This means the distance from the starting point and endpoint of data preference range to $d_{z,m} = \iota_{z}^t$ will be further. In other words, the width of the data preference range will be broader. 
When $o_\text{max}^t$ increases, the parameter $\theta^t$ also increases, leading to $\nu_z^t$ decreases when $d_{z, m} \leq \iota_z^t$ and $\nu_z^t$ increases when $d_{z, m} \geq \iota_z^t$, which means that the $\nu_z^t d_{z, m}$ calculated by \eqref{mp01} is closer to $\iota_z^t$ compared to \eqref{mp}. In summer, the $\nu_z^t d_{z, m}$ will be closer to $\iota_z^t$ when $o_\text{max}^t$ increases, and then the width of the data preference range will be broader as $o_\text{max}^t$. 
On the other hand, as shown in Fig. \ref{fig: DCGain} (b) and (c), when $o_\text{max}^t$ increases, the $\iota^t$ increases, which also means the width of the data preference range is broader in the scarcest category.

Last, as mentioned in \cite{deng2021auction}, the MCs with more categories can achieve a better performance of the global model compared to MCs with fewer categories but the marginal benefit of increasing the number of categories will gradually decrease.
Therefore, the gain of data category diversification of MC $m$ was defined as follows
\begin{equation}
\label{u_cate}
u^\text{c}_m=\mu\sin{\left(\frac{\pi z_m}{2Z}\right)},
\end{equation}
where $\mu$ is the upper bound of data category diversification that needs to be given in advance, and $z_m$ is the number of data categories owned by MC $m$.

\subsection{Energy Consumption Model}
After defining the contribution of each MC to the execution of different local iterations, our scheme designs a cost function for MCs to calculate the cost of executing the FL task.
Due to the heterogeneity of data distribution, device hardware configuration, channel state, and task accuracy requirements, the energy cost of executing the same task is different for different MCs. Similarly, the energy cost of executing tasks with different local iterations is different for the same MC.

\subsubsection{Computation Energy Consumption}
According to the workflow of FL, each MC only needs to complete the required tasks within a fixed time $t_{f}$ seconds. Based on the widely accepted system time model \cite{r5}, the required CPU working frequency (Hz) of MC $m$ to execute task $n$ is given by
\begin{equation}
    \label{cpu_f}
    f_{m}^t = \frac{a_mD_ml_m^t}{t_{f}},
\end{equation}
where $a_m$ represents the number of CPU cycles to process one data sample. Based on \eqref{cpu_f} and according to \cite{r5}, the energy consumption of MC $m$ in Joule required for one local training is obtained as 
\begin{equation}
\label{e_cmp}
E_{m}^{t,\text{cmp}}=\frac{\zeta\left(a_mD_m\right)^3{{l}_m^t}^3}{{t_{f}}^2},
\end{equation}
where $\zeta$ is the effective capacitance parameter of the computing chipset for MC.
Without loss of generality, we assume that the effective capacitance parameter of the computing chipset is the same for each MC.
\subsubsection{Communication Energy Consumption}
Since the structures of the global model and the local models are the same, we use $\varrho$ (bits) to denote the model size. To improve communication efficiency, we assume that all selected clients take the same transmission time to upload updated local models. This necessitates that the same transmission rates $s^t$ (bits/s) for each client in communication round $t$. Due to our focus on incentive mechanism design, the channel is assumed to be slowly fading and stable during each communication round and frequency-division multiple access (FDMA) is adopted as the transmission scheme. Then, according to Shannon's formula, the MC $m$’s communication power consumption in one communication round is given by
\begin{equation}
\label{energycomcost1}
P_{m}^t = \frac{(2^{\frac{s^t}{B_m^t}}-1)B_m^t}{h_m^t},
\end{equation}
where $B_m^t$ (Hz) is the bandwidth allocated to MC $m$. The normalized channel coefficients of sub-consumers for MC $m$ are denoted as $h_m^t = \frac{\hat{h}_m^2}{\psi_0^t}$, which is assumed to be perfectly estimated at the CS. $\hat{h}_m$ is the channel gain from the MC $m$ to the CS (as a base station), and $\psi_0^t$ is the one-sided additive white Gaussian noise (AWGN) power spectral density. Given the amount of transmitted bits $\varrho$, uplink transmission rate $s$,and the transmission power $P_m^t$, the transmission energy consumption for one communication round is
\begin{equation}
\label{energycomcost2}
E_{m}^{t, \text{com}} = \frac{(2^{\frac{s^t}{B_m^t}}-1)B_m^t\varrho}{h_m^t s^t}.
\end{equation}

With energy consumption for computation and communication given \eqref{e_cmp} and \eqref{energycomcost2} respectively, we can calculate the collective energy consumption vector of MC $m$, i.e.,
\begin{equation}
\label{energycost}
E_{m}^t=E_m^{t, \text{com}}+E_{m}^{t, \text{cmp}}.
\end{equation}

\begin{table}[ht]
    \centering
    \caption{Main Symbols}
    \begin{tabular}{c|c}
    \toprule
    Notation & Definition \\
    \midrule
    $B_m^t$                  & Bandwidth       \\
    $c_{m}^t$             & Data quality        \\
	$d_{z, m}, d_\text{avg}$ & Data size of category $z$       \\
    $E_{m}^\text{cmp}, E_{m}^{t, \text{com}}$ & Energy consumption        \\
	$\textbf{g}^t$ &  The amount of trained data \\
    $\hat{h}_m$                  & The channel gain       \\
	$l_m^t$                  &  Local iterations       \\
    $\mathcal{M}$                  & The set of mobile clients        \\
    $\textbf{o}^t$                  & Data category difference        \\
	$P_m^t$                  & Communication power consumption   \\
    ${q}_{m}^t$                  & The winner indicator parameter        \\
	$r_m^t$                  & Preset reward\\
    $s^t$                  & Transmission rates        \\
	$t$                 & Communication round        \\
    $u_{z, m}^t $       & Unit data quality \\
    $U_{m}^t, U_\text{s}^t, U_\text{w}^t$& MC utility, CS utility, social welfare \\
	$v_m^t$                  & Historical selected round      \\
    $\bm{\omega}^t$  & The global model        \\
    ${\iota}_m^t$  & The category reference parameter        \\
    ${\nu}_m^t$  & The gain compensation parameter        \\
    ${\tau}^t, \tau_m^t$                 & The aggregate weight        \\
	$Z, z_m$                  & The number of category     \\
    $\alpha$                  & The high bound of UDQ      \\
    $\Gamma$                  & Probability density function \\
    $\lambda_m^t $                  & The no-bias factor        \\
    $\sigma$                  & The normalized factor        \\
    $\kappa_m^t $                  & Deposits  \\
    \bottomrule
    \end{tabular}
    \label{tab:my_label}
\end{table}

\section{Truthful Auction Design}

The data quality and energy consumption are important criteria for client selection, false bid information will lead to false client selection schemes. Therefore, ensuring the authenticity of information can ensure the effectiveness of client selection.
However, there is typically an inherent information asymmetry between the CS and the MCs in the IoV scenarios. This necessitates the development of an incentive mechanism to motivate MCs to provide accurate and truthful information.
On the other hand, due to the unique datasets possessed by different MCs, their contributions to FL vary, often accompanied by differing energy costs. Additionally, incentivizing user participation in training is necessary. Hence, incentive mechanisms involve aligning the contributions of MCs to the CS with appropriate rewards.
Based on the above, we design a VCG auction-based incentive mechanism.

Within the auction, at the beginning of each communication round, MCs express an interest in FL provide the CS with estimates of the computation energy consumption required for one local iteration and their respective channel coefficients and by utilizing this submitted information, coupled with the data quality for each MC assessed by the CS in the \eqref{dataquality_2}, the CS joint optimize the numbers of local iteration, bandwidth allocation, and client selection of FL to maximize overall social welfare, all the while upholding the principles of honesty among participating MCs and ensuring that each mobile client attains a non-negative utility. The above two goals can be achieved by the above joint optimization.
\subsection{Preset Reward}
The preset rewards will be issued after the local model has been uploaded to CS. If an MC $m$ participates in a round of FL, it can get a preset reward, i.e.,
\begin{equation}
    \label{taskreward}
    r_m^t = q_m^t r_{\text{0}},
\end{equation}
where $r_{\text{0}}$ is the basic reward for one local iteration. However, the participating MC's actual reward is the preset reward minus the deposit previously submitted to the server, based on the LCSFLA workflow.
In Section V-C, we will elaborate on how we design the deposit based on the preset reward to ensure the utility of MC is non-negative and align individual MC interests with collective interests, aiming to incentivize user participation and truthful bidding.


\subsection{Social Welfare}
In this subsection, according to the data quality, energy cost, deposit, and preset reward mentioned above, we will define the utility of each component of the incentive mechanism in the FL system.
\subsubsection{MC Utility Evaluation}
Following the operational procedures of the incentive mechanism, the gain of the participant MC $m$ is a preset reward ${r}_m^t$ from the CS while its cost including the energy consumption ${E}_m^t$ during training and the deposit $\kappa_m^t$ need to submit to the CS.
Thus, the utility of MC $m$ in round $t$ is expressed as
\begin{equation}
\label{mc_utility}
U_m^t={{{q}}_m^t} {({r}_m^t-{E}_m^t)}-\kappa_m^t,
\end{equation}
\subsubsection{CS Utility Evaluation}
The gain of the CS is the model contributions of the participant MC, i.e., the sum of the data quality of each MC. In addition, the cost of the CS including the deposit submitted by MCs minus the reward paid to MCs.
Thus, the utility of the CS in round $t$ is expressed as
\begin{equation}
\label{cs_utility}
U_\text{s}^t=\sum_{m\in \mathcal{M}}{{{{{q}}_m^t}({c}_m^t- {r}_m^t)}+\kappa_m^t},
\end{equation}
where the data quality ${c}_m^t$ of MC $m$  was calculated by using the \eqref{dataquality_2}. 
\subsubsection{Social Welfare}
The total social welfare $U_\text{w}^t$ refers to the sum of the utilities of each component in the incentive mechanism, i.e., the utility of all MCs and the CS, as follows
\begin{equation}
\label{Socialwelfare}
U_\text{w}^t=\sum_{m\in \mathcal{M}}{{q}_m^t}{({c}_m^t- {E}_m^t)}.
\end{equation}
\subsection{Truthfulness Auction}
\subsubsection{Desired Economic Properties}
In scenarios characterized by information asymmetry,  {the incentive mechanism must have} several key features. These features are essential to incentivize MCs to participate in FL and provide accurate and truthful information.
 {
\begin{theorem}[Individual Rationality]:
For each MC, the benefit of participating in FL must be non-negative, i.e.,
\begin{equation}
    \label{IR}  
    U_\text{s}^t={{{q}}_m^t} {({r}_m^t-{E}_m^t)}-\kappa_m^t \geq 0.
\end{equation}
\end{theorem}
\vspace{-0.25cm}
\begin{proof}
The detailed proof is presented in the Appendix B
\end{proof}}
\vspace{-0.25cm}
 {
\begin{definition}
\label{d1}
    We use $\widetilde{U}_{m}^t={\widetilde{q}_{m}^t}{(\widetilde{r}_m^t-\widetilde{E}_m^t)}-\widetilde{\kappa}_{m}^t$ to represent the utility that MC $m$ gets when it uploads untruthful bidding, where ${{\widetilde{{q}}}_{m}^t}$, $\widetilde{r}_m^t$ and $\widetilde{\kappa}_{m}^t$ represents the winner indicator vector, the preset reward, and the deposit when MC $m$ uploads untruthful bidding, respectively. 
\end{definition}}

 {
\begin{theorem}[Incentive Compatibility]:
The designed incentive mechanism must guarantee that the behavior that enables the MC to pursue personal interests is consistent with the goal of maximizing the collective value, i,e.,
\begin{equation}
    \label{IC}
        \begin{split}
        U_m^t \geq \widetilde{U}_{m}^t.
        \end{split}
    \end{equation}
\end{theorem}
\begin{proof}
The detailed proof is presented in the Appendix C
\end{proof}}
\subsubsection{Deposits Determination}
Based on the workflow introduced in Section III-A, the establishment of the deposit mechanism is the key work to ensuring individual rationality and incentive compatibility of the incentive mechanism. This ensures that truthful bidding becomes the dominant strategy for each MC.
After determining the winner indicator vector introduced in Section IV-D, the CS will collect a deposit from the MC that wins the FL task.
According to the explanation of the payment mechanism in the VCG auction in \cite{rVCG}, the payment of a participant buyer $m$ is described as the difference between the maximal social welfare when buyer $m$ did not participate the auction and the maximal social welfare excluding the utility of buyer $m$ when such buyer $m$ participant the auction. In summary, buyers are required to cover a portion of the loss in overall benefit resulting from their participation in the auction.
Since ${c}_m^t$ and ${r}_m^t$ can be computed at CS, plus to satisfy \eqref{IC}, the formulation of $\kappa_m^t$ is as follow
\begin{eqnarray}\label{pm}
        \kappa_m^t=\sum_{m\in \mathcal{M}_{-m}}{{\hat{q}}_{m}^{t}} \left(\hat{c}_m^t - \hat{E}_m^t\right) - {{{q}}_m^t}({c}_m^t-{r}_m^t)
        \nonumber\\
        -\sum_{m\in \mathcal{M}_{-m}}{{{q}}_m^t} \left({c}_m^t-
        {E}_m^t\right),
\end{eqnarray}
where ${{\hat{q}}_{m}^{t}}$, $\hat{c}_m^t$ and $\hat{E}_m^t$ represents the winner indicator vector, data quality, and energy consumption obtained in the process of solving the social maximization welfare after eliminating MC $m$ from the MC set $M$, respectively. The $\mathcal{M}_{-m}$ stands for the set of MC including all MC except the MC $m$.
According to the information mentioned above, the proof of IR and IC will be shown in Appendix B and C, respectively.
\subsection{Problem Formulation}
 {In each communication round, the CS performs client selection to determine the winning MC based on the bids collected from the MCs to most effectively reduce the DCD gap across various categories in the following multiple communication rounds.}
However, since we need to achieve a balance between the data quality and energy consumption costs of MCs, the local iterations and bandwidth allocated to the MCs need to be considered in the client selection.
The increasing local iterations lead to higher data quality but also result in higher computation energy costs. Therefore, balancing computation energy cost and data quality is an essential part of determining the optimal client selection to address non-IID issues.
Additionally, since the bandwidth of CS is limited, we need to select the combination of MCs and allocate the bandwidth across different MCs to maximize benefits while satisfying the bandwidth constraint. For an MC with poor channel conditions, possessing higher data quality or achieving lower computation energy consumption increases its chances of being selected.
In summary, the joint optimization of bandwidth allocation, local iterations, and client selection for FL in the wireless network achieves maximum social welfare.
Based on the information mentioned above, the social welfare maximization (SWM) problem is formulated as follows

\begin{subequations}
    \label{WD}
    \begin{align}
        \left(\text{SWM}\right)~~&\max_{\{\bm{q}^t, \bm{B}^t, \bm{l}^t\}}
        \label{main01}{\sum_{m \in \mathcal{M}}{{q}_m^t}}\left({c}_m^t - E_m^t\right)\\
        ~{\rm s.t.}~&{\rm constraints ~in~}\eqref{task restrain 2} 
        \nonumber\\
        &~\label{010}q_{m}^t\in\{0,\ 1\},\forall m\in \mathcal{M},\\
        &~\label{012}l_{m}^t\in \{ 0,\dots, l_\text{max}\},\forall m\in \mathcal{M},\\
        &\label{011}\sum_{m\in M}q_{m}^tB_m^t \leq B_\text{max},
    \end{align}
\end{subequations}
where $l_{\text{max}}$ is the maximum number of iterations allowed. 
The problem \eqref{WD} is a non-convex problem. The presence of discrete variables $l_m^t$ and $q_m^t$ and the coupling of the winner indicator variables $q_m^t$ with other variables, including the energy costs $E_m^t$ related to the bandwidth $B_m^t$ and the data quality $c_m^t$ related to the local iterations $l_m^t$ in the object function as well as the bandwidth $B_m^t$ in the constraints, make the problem difficult to solve. To convexify them and obtain the global optimal solution of problem \eqref{WD}, we need to reformulate the coupled terms into addition forms, followed by the branch and bound method to deal with the issue of discrete variables. 

According to \cite{tan2021robust}, the following proposition is proposed: Assume that $a > 0$ is a positive real number. Then the bilinear set $\mathcal{B} = \{(x, y, w) \in \mathbb{R}_+ \times  {\mathbb{Z}_+}\times  \mathbb{R} : w = xy, x \leq a, y \leq 1\}$ is equivalent to the linear set $\mathcal{M} = \{(x, y, w)\in \mathbb{R}\times  {\mathbb{Z}} \times  \mathbb{R}:w \geq 0, w \geq x+ay-a, w\leq ay, w \leq x\}$. This proposition establishes an equivalence between a bilinear set, where variables are in the form of multiplication, and a linear set, where variables are in the form of addition. 
Before using this proposition, problem \eqref{WD} needs to be reformulated as follows
\begin{subequations}
    \label{WD_1}
    \begin{align}
        \max_{\{\bm{q}^t, \bm{B}^t, \bm{l}^t,\bm{\chi}^t, \bm{\varphi}^t, \bm{\gamma}^t\}}~~&
        \label{main03}{\sum_{m \in \mathcal{M}}{\varphi}_{m}^t\left(\check{c}_m^t- e_\text{1}{{\varphi}_{m}^t}^2 \right) - e_\text{2}{\gamma}_m^{t}} \\
        ~{\rm s.t.}&~{\rm constraints ~in~}\eqref{task restrain 2}, \nonumber\\
        &~\label{030}q_{m}^t\in\{0,\ 1\},\forall m\in \mathcal{M},\\
        &~\label{031}l_{m}^t\in \{ 0,\dots, l_\text{max}\},\forall m\in \mathcal{M},\\
        &~\label{032}{\varphi}_{m}^t={{q}_{m}^t}l_m^{t},\forall m\in \mathcal{M},\\
        &~\label{033}{\varsigma}_{m}^t={{q}_{m}^t}B_m^{t},\forall m\in \mathcal{M},\\
        &~\label{034}{\gamma}_{m}^t={{q}_{m}^t}{\chi}_m^{t},\forall m\in \mathcal{M},\\
        &\label{035}\sum_{m\in M}{\varsigma}_{m}^t \leq B_\text{max},\\
        &~\label{036}{\chi}_m^{t} \geq (2^{{s^t}/{B_m^t}}-1)B_m^t,
    \end{align}
\end{subequations}
where $\check{c}_m^t= \sigma\lambda_m^t \sum\nolimits_{z=1}^{Z}  u_{z, m}^t d_{z,m} $, $e_\text{1}={\zeta\left(a_mD_m\right)^3}/{{t_{f}}^2}$, and $e_\text{2}={\varrho}/\left({h_m^t s^t}\right)$. Since $q_m^t$ is a binary variable, we have $q_m^t={q_m^t}^2$. Therefore, $q_m^t{l_m^t}^3={q_m^t}^3{l_m^t}^3$. 

 {At this point, since constraints \eqref{032} to \eqref{034} are bi-convex, the problem \eqref{WD_1} is not a convex problem.
By applying this proposition to \eqref{032}-\eqref{035}, problem \eqref{WD_1} can be converted into a convex problem. For this purpose, we assume that ${\gamma}_{m}^t$, ${\varsigma}_{m}^t$ and ${\varphi}_{m}^t$ are viewed as $``w"$, and ${\chi}_m^{t}$, $B_m^{t}$ and $l_m^{t}$ are viewed as $``x"$ when ${q}_{m}^t$ was viewed as $``y"$. Then we can equivalently transform \eqref{032}-\eqref{035} into linear constraints as}
\begin{equation}
    \label{cons1}
    \begin{split}
        \bm{A}_\text{1} [{\gamma}_{m}^t, {\chi}_m^{t}, q_{m}^t, 1 ]^T \preceq \bm{0},\\
        \bm{A}_\text{2} [{\varsigma}_{m}^t, B_m^{t}, q_{m}^t, 1 ]^T \preceq \bm{0},\\
        \bm{A}_\text{3} [{\varphi}_{m}^t, l_m^{t}, q_{m}^t, 1 ]^T \preceq \bm{0},
    \end{split}
\end{equation}
 {where the parameter matrices are}
 \vspace{-0.0ex}
$$
~\bm{A}_\text{1}= \begin{bmatrix}
       1, &  0, &  -l_\text{max}, &  0 \\
       -1, &  0, &  0, &  0 \\
       1, & -1, &  0, &  0 \\
       -1, &  1, &  l_\text{max}, &  -l_\text{max} 
\end{bmatrix},
$$ 
$$
~\bm{A}_\text{2}= \begin{bmatrix}
       1, &  0, &  -B_\text{max}, &  0 \\
       -1, &  0, &  0, &  0 \\
       1, & -1, &  0, &  0 \\
       -1, &  1, &  B_\text{max}, &  -B_\text{max} 
\end{bmatrix},
$$
$$
~\bm{A}_\text{3}= \begin{bmatrix}
       1, &  0, &  -\chi_\text{max}, &  0 \\
       -1, &  0, &  0, &  0 \\
       1, & -1, &  0, &  0 \\
       -1, &  1, &  \chi_\text{max}, &  -\chi_\text{max} 
\end{bmatrix}.
$$
 {The $\chi_\text{max}$ is the upper bound of the sum of $\chi_m^t$.}

 {In addition, since the variables $l_m^t$ and $q_m^t$ are discrete, the problem \eqref{WD_1} is the non-convex problem. Therefore, we need to relax two discrete constraints \eqref{030} and \eqref{031} into continuous constraints, as follows}
\begin{subequations}
    \label{consist}
    \begin{align}
        &\label{cs030} 0 \leq q_{m}^t \leq 1,\forall m\in \mathcal{M},\\
        &\label{cs031} 0 \leq l_{m}^t \leq l_\text{max},\forall m\in \mathcal{M}.
    \end{align}
\end{subequations}
On this basis, problem \eqref{WD_1} can be reformulated as
\begin{subequations}
    \label{WD01}
    \begin{align}
        \max_{\{\bm{q}^t, \bm{B}^t, \bm{l}^t,\bm{\chi}^t, \bm{\varphi}^t, \bm{\gamma}^t\} }
        \label{main02}~{\sum_{m \in \mathcal{M}}{\varphi}_{m}^t\left(\check{c}_m^t- e_\text{1}{{\varphi}_{m}^t}^2 \right) - e_\text{2}{\gamma}_m^{t}} \\
        {\rm s.t.}~{\rm constraints ~in~} \eqref{task restrain 2}, \eqref{036},\eqref{cons1}, \eqref{cs030},\eqref{cs031},\nonumber
    \end{align}
\end{subequations}
There are numerous well-studied deterministic methods for finding the global solutions to the convex problem, such as the interior point method. 
However, problems \eqref{WD} and \eqref{WD01} are not equivalent. To obtain the solution of problem \eqref{WD}, we need to convert the continuous values $q_m^t$ and $l_m^t$ back into discrete values. We can use the branch-and-bound method to determine these discrete values. Upon completing the branch and bound method, we will obtain the optimal winner indicator vector $\bm{q}^t=[{{q}}_1^t, \cdots, {{q}}_M^t]\in\mathbb{Z}^{M}$ and local iterations vector $\bm{l}^t=[{{l}}_1^t, \cdots, {{l}}_M^t]\in\mathbb{Z}^{M}$. In this process, each branch corresponds to a discrete value of $q_m^t$ or $l_m^t$, and the interior point method is used to decide whether to prune the branch. By the final step of determining the last discrete variable in the branch-and-bound method, we can obtain the value of the optimal bandwidth allocation vector $\bm{B}^t=[B_1^t, \dots, B_M^t]\in\mathbb{R}^{M}$.
Based on the optimal winner indicator vector, the selected MCs will participate in the FL. The training workflow is illustrated in Algorithm 1.
\begin{algorithm}
\caption{The workflow of LCSFLA}\label{alg:cap}
\hspace*{0.02in} {\bf Input:} 
Set of all MC $\mathcal{M}$, the initial global model $\bm{\omega}^{0}$,
and the amount of Winner $N$.
\\
\hspace*{0.02in} {\bf Output:} 
The global model $\bm{\omega}^T$
\begin{algorithmic}[1]
\STATE {A learning task is submitted to the CS (FL platform);}
\STATE {The MCs that want to participate in FL submit their own local data distribution to the CS;}
    \FOR {$t=0$ to $T$}
    	\STATE {The CS broadcast $\bm{\omega}^t$ and task to all MC that want to participate in FL;}
        \FOR{each MC $m \in \mathcal{M}$}
            \STATE {The MC $m$ computer its energy cost according to the equation. \eqref{cpu_f} - \eqref{energycost};}
            \STATE {The MC $m$ submit cost and channel information to CS;}
        \ENDFOR
        \STATE {The CS calculates the data quality of MCs according to the equation. \eqref{dataquality_2};}
        \STATE{The CS selects a subset $M^t$ of MCs as winners for the current communication round of FL task participation to maximize social welfare in each round according to the \eqref{WD};}
        \FOR{each MC $m \in M^t$}
            \STATE {The MC $m$ trains on local data set $\mathcal{D}_m$;}
            \STATE {The MC $m$ updates model weights $\bm{\omega}^{t-1}_m$ according to the \eqref{local-updata};}
            \STATE {The MC $m$ submit $\bm{\omega}^{t}_m$ to CS; }
        \ENDFOR
    \STATE The CS calculates $\bm{\omega}^{t}$ according to the \eqref{glabel_updata}
    \ENDFOR
\end{algorithmic}
\end{algorithm}

The above solution to the problem constitutes the Social-Welfare-Maximization (SWM) problem. Algorithm 1 summarizes the main steps and pseudocode of the algorithm. The computational complexity of solving the SWM, that is, the branch-and-bound algorithm problem, is $O(M)$. In addition, when using \eqref{pm} to determine the deposit $\kappa_m^t$ of MC $m$, the SWM problem needs to be solved again to obtain the $\hat{{\bm{{q}}}^{t}}$.
To determine the deposits of all participant MCs, the matching problem needs to be solved up to $N$ times. Therefore, the total computational complexity of the socially optimal auction implemented on the CS is $O((1+N)M)$.

\section{SIMULATION RESULTS}
\subsection{Experiment Setting}
\subsubsection{Experiment Environment}
This article established a Python 3.9 software environment based on Pytorch, and the hardware environment is a computer with 2.30 GHz intel Core i7-11800H 8-core Processor CPU, 16.00 GB, and Win10 64 bit, NVIDIA GeForce RTX 3060 SUPER.
\subsubsection{Experimental Datasets and Models}
In the simulation experiments, we adopt four datasets commonly used in other works. The first dataset is CIFAR-10, which includes 50,000 training images and 10,000 test images in 10 categories. The second dataset is FASHION-MNIST, which includes 60,000 training images and 10,000 test images from 10 categories, with 7,000 images per category. The third dataset is MNIST, which is a handwritten character dataset consisting of 10 categories, with 60,000 images. The last dataset is the Traffic Sign Recognition Database (TSRD) \cite{tsrd}, which includes 4,170 images training images and 1,994 test images from 3 categories. In addition, we trained three different nerve models corresponding to four datasets, namely CNN\footnote{The CNN for MNIST has the following structure: 5×1×32 Convolutional → 2×2 MaxPool → 5×1×2 Convolutional → 2×2 MaxPool → Flatten → 3136×512 Fully connected →dropout → 512×10 Fully connected → softmax } for MNIST, 2NN\footnote{The 2NN for Fashion MNIST has the following structure: 716×200 Fully connected→ 200×200 Fully connected→200×10 Fully connected→ softmax} for Fashion MNIST and Resnet18 \cite{he2016deep}  for CIFAR-10 and TSRD.

\subsubsection{Data Distribution Setting}
For data distribution experimental setup, we have set up two different experimental environments: $\mathbf{Case\ 1}$: distributes data with the method of Dirichlet distribution to the MC,  $\alpha$ = 0.3; $\mathbf{Case\ 2}$: 
The MC $m$ initiates the process by uniformly choosing a number $z_m$ from the range of 1 to 3, 1 to 6, or 1 to 9, determining the number of categories within its dataset.
Subsequently, MC $m$ randomly selects $z_m$ categories from the available categories, which comprises ten category labels in the case of the MNIST and Cifar-10 datasets. Next, MC $m$ uniformly chooses a number $d_{z, m}$ from the range of 10 to 200 or 10 to 100, ascertaining the data size of category $z$.
In this scenario, the number of data sizes, denoted as $D_m$ for MC $m$, ranges from 10 to 2000.
\subsubsection{Benchmark Mechanism}


The LCSFLA will be compared with three benchmark algorithms. 
$\mathbf{1)\ FedAvg}$\cite{mcmahan2017communication}: In each communication round, CS will randomly select $N$ MCs to participate in FL.  {$\mathbf{2)\ CSFedAvg}$\cite{zhang2021client}: In each communication round, CS will randomly select $0 - N$ MCs as the random MC set, and choose one MC whose local model is closest to the global model to add to the candidate MC set. The final selected MC set, i.e., the MCs who will participate in the update of the global model, consists of the candidate MC set and the random MC set. In addition, the maximum size of the candidate MC set is set to $N/2$. $\mathbf{3)\ CAFL}$\cite{lu2023auction}: All MCs that want to participate in FL need to process local training based on the same initial global model. Then, Using K-means clustering all MCs into $N/2$ groups based on the local training results, two MCs with the optimal bids, obtained by the Nash equilibrium, are selected from each group. Finally, we can get $N$ MCs to participate in FL in each communication round.
$\mathbf{4)\ FL-DPS}$\cite{zhang2023dpp}: Before FL begins, local training is required for all MCs wishing to participate. This is followed by the calculation of a similarity matrix based on the Fully Connected layer of the locally trained models. Subsequently, the probability of selecting each client combination of size $N$ is computed using K-DDP \cite{kulesza2011k}. Based on these probabilities, a client combination is selected for FL in each communication round.
$\mathbf{5)\ LCSFLA-bias}$: Then we will compare the performance when the ``no biased" factor $\beta$ changes based on the history train rounds (LCSFLA) and the performance when the ``no biased" factor $\beta$ is set to one (LCSFLA-bias). }
\subsubsection{Simulation Parameters}
The simulation parameters are configured as follows: 
The number of participant MC denoted as $M$, is 100. The number of selected MCs denoted as $N$, is set to 10 for benchmark algorithms. For TSRD, due to less train data size, the $M$ is set to 20, and $N$ is set to 4. The mobile clients were required to finish the local data train within 1 minute. The total number of communication rounds is set to 50 or 100. The basic reward for one local iteration is set to 200. The special super-parameter settings in this paper are as follows: $\alpha = 2$, $\mu=0.2$, $\vartheta=10, \beta = 0.95, \sigma = 1$.
Parameter $a_m$ is 2 megacycles/samples, and the effective capacitance parameter of the computing chipset is $\zeta = 10^{-28}$ \cite{kang2019incentive}. The model size, denoted as $\varrho$, is 8 Mbits, and the transmission rate for uploading the local model, denoted as $s$, is 2 Mbits/s. We assume that the noise power spectral density level ${\psi_0}$ is -130 dBm/Hz, and the dynamic range of the channel power gain $\hat{h}_m^t$ varies from -90 dB to -100 dB \cite{jiao2020toward}. Therefore, we uniformly generate MC $m$'s normalized channel power gain $h_m^t$ within the range of $[10^6, 10^7]$. The total available bandwidth is $B_{\text{max}} = 10$ MHz, and the total energy consumption is constrained to be below 40 W. 
The other parameters of the simulation are referred to \cite{kang2019incentive}.

\begin{figure}[ht]
    \centering
    \includegraphics[width=0.45\textwidth]{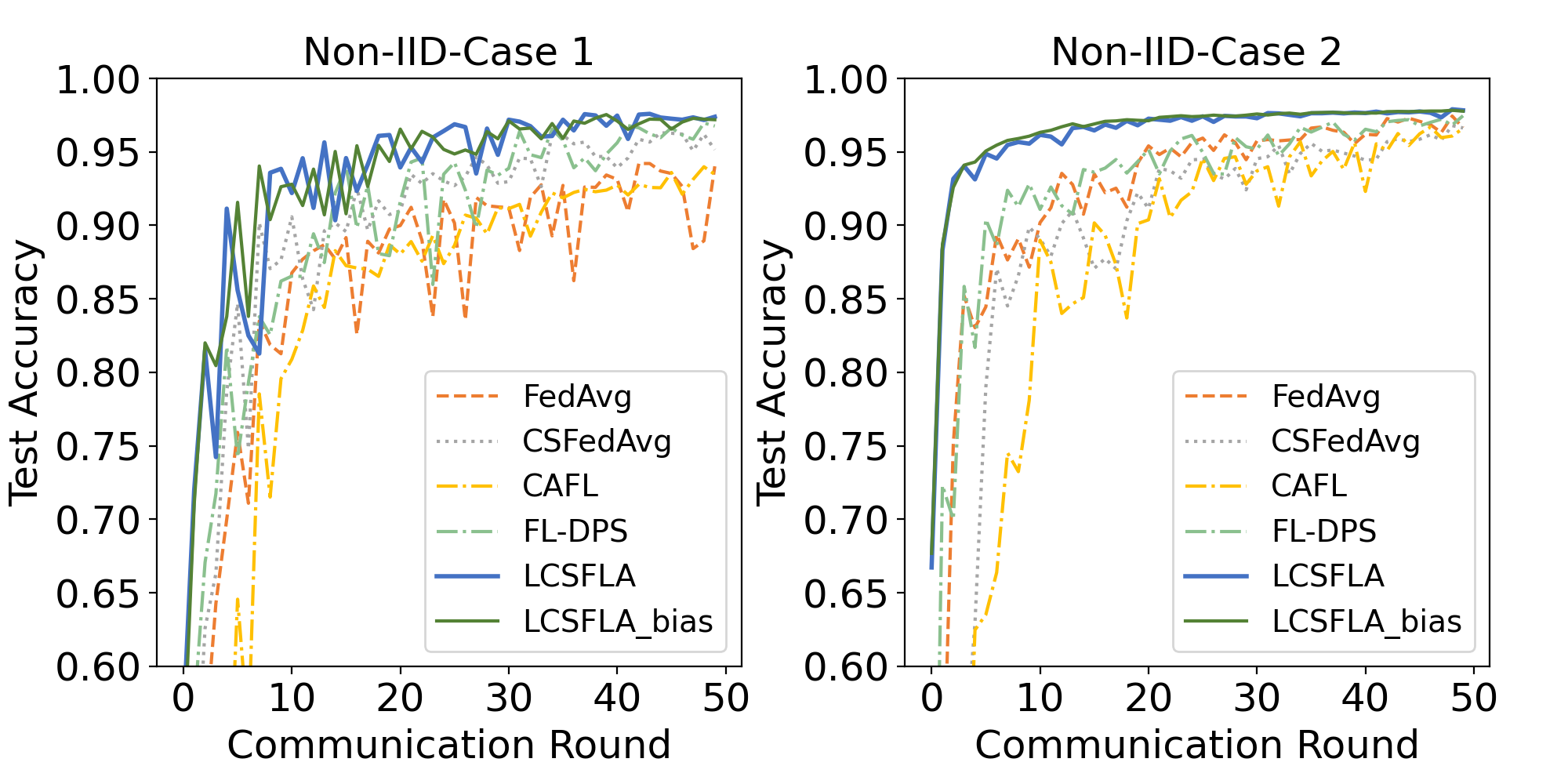}
    \caption{ {The accuracy performance of CNN model on MNIST dataset}}
    \label{mc}
\vspace{-0.5cm}
\end{figure}
\begin{figure}[ht]
    \centering
    \includegraphics[width=0.45\textwidth]{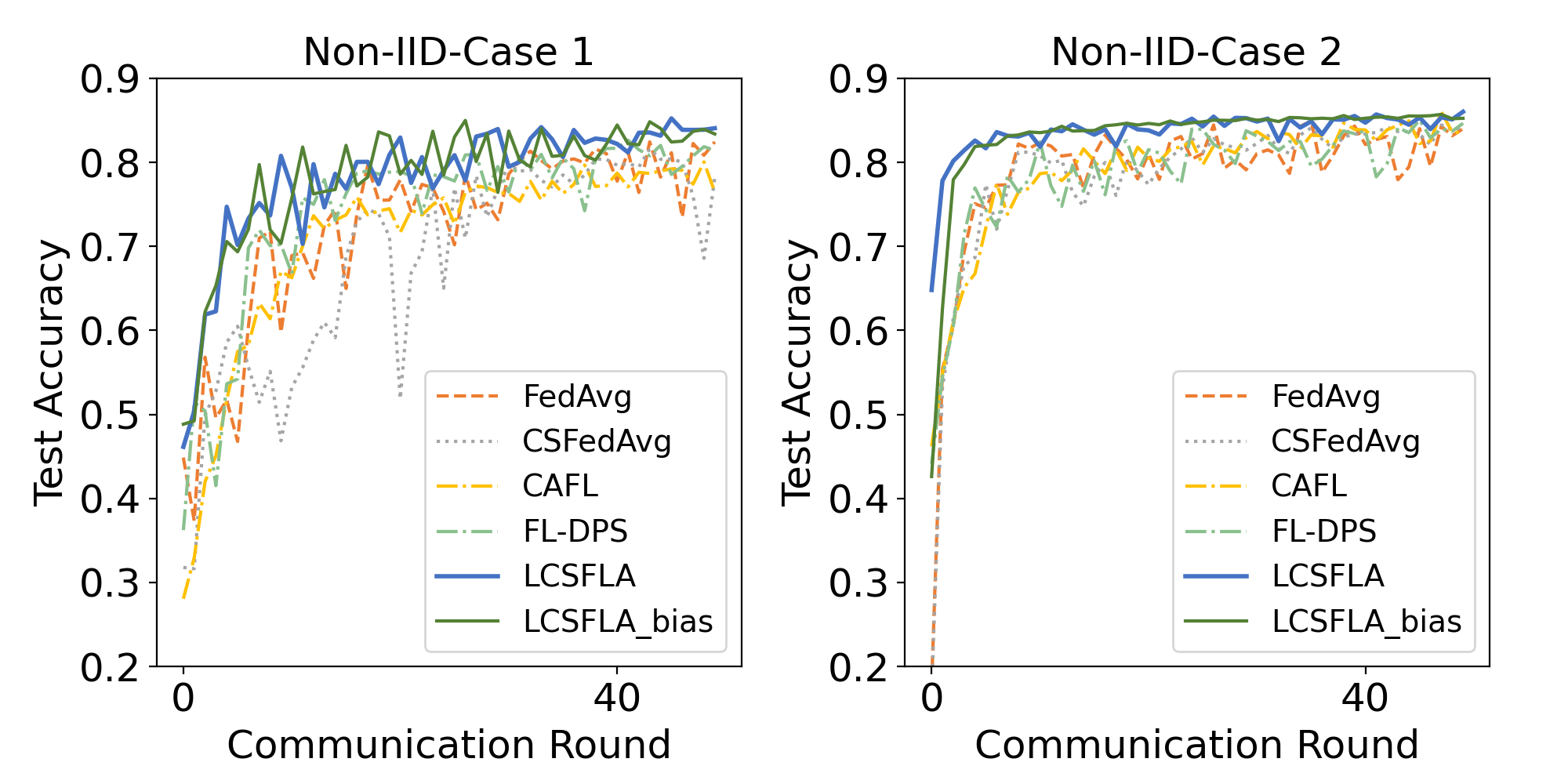}
    \caption{ {The accuracy performance of 2NN model on FASHION MINST dataset}}
    \label{fmc}
    \vspace{-0.5cm}
\end{figure}
\begin{figure}[ht]
    \centering
    \includegraphics[width=0.45\textwidth]{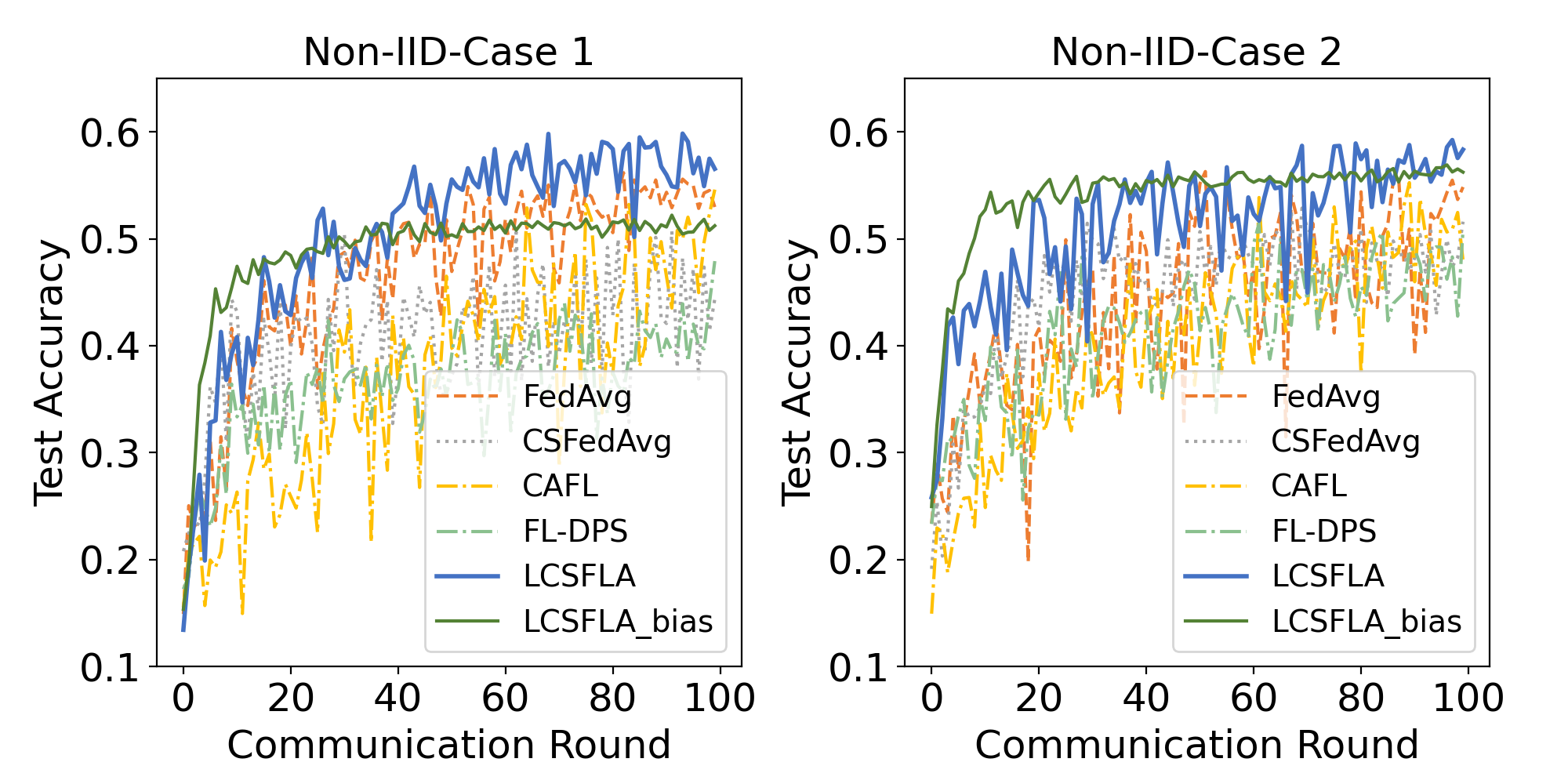}
    \caption{ {The accuracy performance of RESNET18 model on CIFAR10 dataset}}
    \label{cc}
    \vspace{-0.5cm}
\end{figure}
\begin{figure}[ht]
    \centering
    \includegraphics[width=0.45\textwidth]{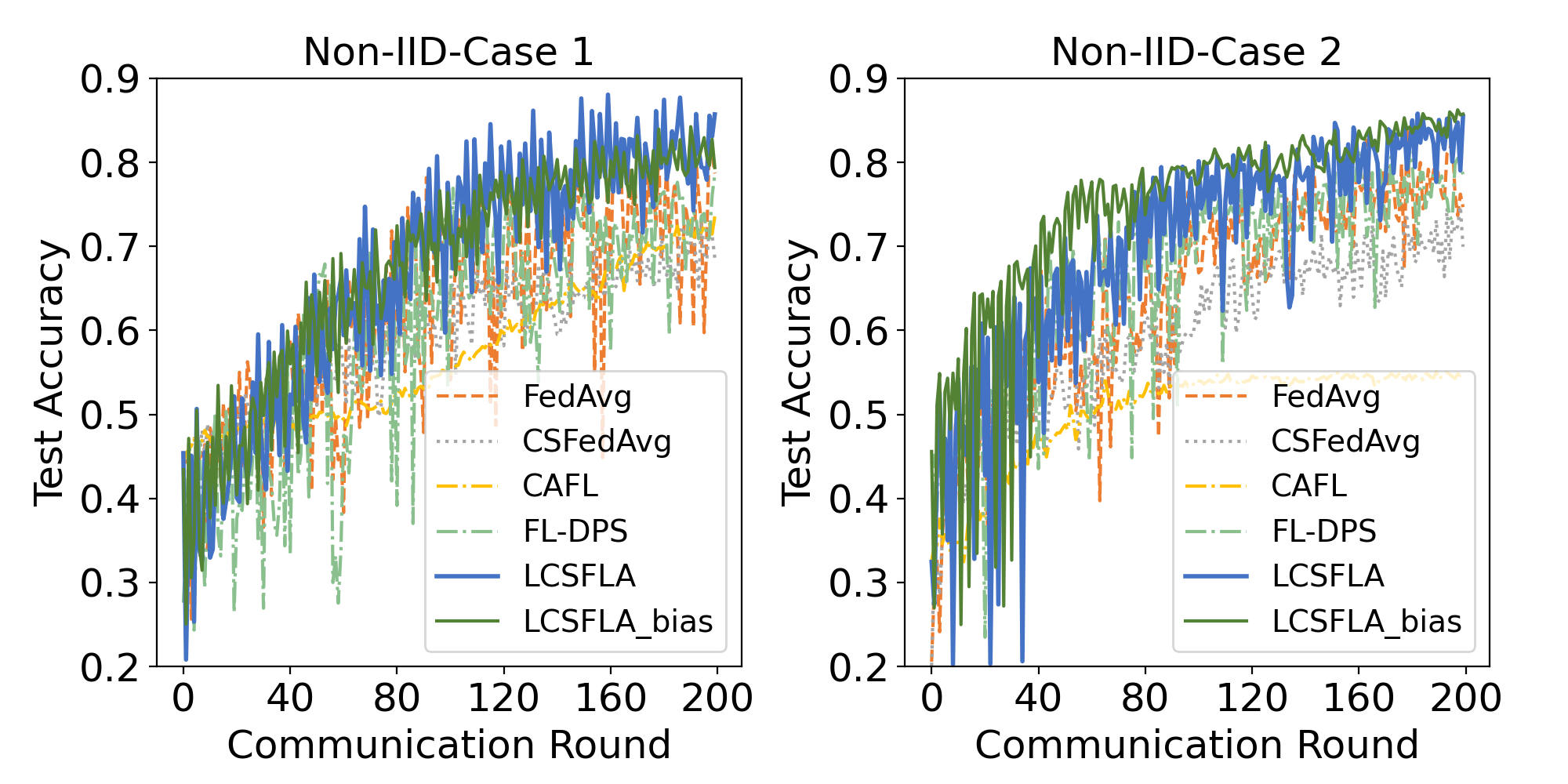}
    \caption{ {The accuracy performance of RESNET18 model on TSRD dataset}}
    \label{tc}
    \vspace{-0.5cm}
\end{figure}
\subsection{Performance of Accuracy Analysis}
Please note that the performance metric in this section is the convergence speed, which is represented by the relationship between test accuracy and communication rounds. However, the comparison scheme does not involve the optimization of communication resources. To ensure the fairness of the experiment, the experiment treats communication consumption as a fixed value. Then, we primarily tested these methods on the MNIST, Fashion-MNIST,  CIFAR-10, and TSRD datasets, under two heterogeneous data distribution settings.
\begin{table*}[htp]
\caption{ {\textbf{PERFORMANCE EVALUATION OF ACCURACY}}}
\centering
\resizebox{\textwidth}{!}{
\begin{tabular}{cccccccccccccc}
\toprule
\multirow{2}{*}{Datasets}&\multirow{2}{*}{ToA@x} & \multicolumn{6}{c}{Case 1}&\multicolumn{6}{c}{Case 2} \\
\cmidrule(r){3-8}
\cmidrule(r){9-14}
&&LCSFLA&LCSFLA-bias&FedAvg&CSFedAvg&CAFL& {FL-DPS}&LCSFLA&LCSFLA-bias&FedAvg&CSFedAvg&CAFL& {FL-DPS} \\
\midrule
&ToA@90 &\textbf{9} &\textbf{8}   &25 &19 &27 &15
        &\textbf{3} &\textbf{3}   &11 &13 &20 &8      \\
\cmidrule(r){2-14}
{MNIST}&ToA@95  &\textbf{19}  &\textbf{21}   &NaN &35  &NaN   &35 
                &\textbf{8}   &\textbf{6}    &25  &36  &42    &23
\\
\cmidrule(r){2-14}
&Accuracy   &\textbf{0.97218}  &\textbf{0.97026} &0.92376 &0.95525  &0.92905 &0.96362
            &\textbf{0.97689}  &\textbf{0.97744} &0.96827 &0.95747  &0.95621 &0.96958
\\

\midrule
&ToA@75 &\textbf{10} &\textbf{11} &18 &30 &27 &16
        &\textbf{3}  &\textbf{3}  &6  &8  &9  &8\\
\cmidrule(r){2-14}
{FASHION MNIST}&ToA@80  &\textbf{17} &\textbf{19} &32 &39 &NaN &27
                        &\textbf{6}  &\textbf{5}  &9  &9  &15  &18
\\
\cmidrule(r){2-14}
&Accuracy&\textbf{0.83461} &\textbf{0.83363}  &0.79627 &0.78419 &0.78453 &0.81083 
         &\textbf{0.85238} &\textbf{0.85408}  &0.82645 &0.83639 &0.83763 &0.82990
\\

\midrule
&ToA@40 &\textbf{15} &\textbf{6}  &16  &10 &53 &52
        &\textbf{4}  &\textbf{4}  &20  &16 &30 &23
\\
\cmidrule(r){2-14}
\multirow{2}{*}{CIFAR10}&ToA@45 &\textbf{16} &\textbf{10} &30 &65 &30 &NaN 
                                &\textbf{16} &\textbf{6}  &22 &52 &30 &48\\
\cmidrule(r){2-14}
&ToA@50 &\textbf{26} &\textbf{35} &41 &NaN &76 &NaN 
        &\textbf{20} &\textbf{9}  &49 &65  &89 &NaN \\
\cmidrule(r){2-14}
&Accuracy   &\textbf{0.56725} &\textbf{0.51130} &0.54475 &0.43848 &0.49447 &0.41555
            &\textbf{0.56891} &\textbf{0.56303} &0.50269 &0.48308 &0.49813 &0.47635\\

\midrule
&ToA@70 &\textbf{87}  &\textbf{85}  &95   &187  &183  &82
        &\textbf{62}  &\textbf{42}  &163  &97   &NaN  &63
\\
\cmidrule(r){2-14}
{TSRD}&ToA@80&\textbf{150}  &\textbf{155} &198  &NaN  &NaN  &NaN
             &\textbf{144}  &\textbf{106} &194  &NaN  &NaN  &NaN
\\
\cmidrule(r){2-14}
&Accuracy   &\textbf{0.81496} &\textbf{0.81334} &0.73506 &0.69915 &0.72202 & 0.73355
            &\textbf{0.83310} &\textbf{0.84530} &0.77553 &0.71023 &0.54659 & 0.77272\\

\bottomrule
\end{tabular}}
\label{table-acc}
\vspace{-0.25cm} 
\end{table*}

Fig. \ref{mc} to Fig. \ref{tc} show the model convergence speed in different datasets using different models. Fig. \ref{mc} displays the convergence speed of the CNN model on the MNIST dataset, Fig. \ref{fmc} displays the convergence speed of the 2NN model on the Fashion MNIST dataset, Fig. \ref{cc} shows the convergence speed of the ResNet18 model on the CIFAR10 dataset while Fig. \ref{tc} shows the convergence speed of the ResNet18 model on the TSRD dataset.
As we observe, both the LCSFLA and LCSFLA-bias algorithms can accelerate the model convergence speed. 

Due to the LCSFLA-bias algorithm setting the unbiased factor $\beta$ to 1, it may perform poorly in the final model performance, especially on the CIFAR10 dataset, where the model accuracy on the test set reaches a bottleneck after 100 communication rounds. This occurs because our auction selection favors MCs with higher contribution evaluation benefits. Without an unbiased factor, MCs with smaller data have little chance of being selected, and the model misses the opportunity to learn new data, which could decrease generalization performance.
We also observed a substantial difference in the performance of CAFL between the two distribution settings. In Case 2 settings, there are only six distinct data distributions, and the clustering algorithm can effectively group the MCs. However, clustering MCs is extremely challenging in Case 1, as the data distribution follows a Dirichlet distribution. After the clustering strategy fails, the model's performance inevitably deteriorates.
However, we can observe that in Case 2, for datasets with a low learning difficulty, since a satisfactory model performance can be achieved without requiring extensive data, the final model performance of LCSFLA-bias is slightly better than that of LCSFLA.

In terms of the number of communication rounds required to achieve the target accuracy, we evaluated the performance of the proposed LCSFLA and LCSFLA-bias algorithms and three baseline algorithms. Similar to \cite{zhang2021client}, we recorded the expected communication rounds to reach the target accuracy, denoted as $ToA\char"40 x$, where $x$ represents the target accuracy. And, we measured the average accuracy in the last ten global iterations, denoted as the final accuracy in Tab.\ref{table-acc}.
From this table, we can see the specific number of communication rounds required to achieve the target accuracy. It is important to note that ``NaN" in Tab.\ref{table-acc} indicates that the target accuracy could not be achieved within a limited number of communication rounds. Specifically, in the CIFAR10 dataset, CAFL cannot achieve 50\% accuracy in the Case 1 setting.

Similarly, energy efficiency aspects were considered to illustrate the effectiveness of the design solution. The effectiveness of the design solution is verified by comparing the unit energy consumption required to achieve a given test accuracy. An efficient client selection algorithm can achieve the same model accuracy with less energy consumption. In practice, this means that we can perform the FL task in a more energy-efficient manner without sacrificing model performance, thus improving the sustainability and energy efficiency of the system.

To show the energy efficiency performance in different client selection schemes, we recorded the sum of energy costs to reach the target accuracy. It is important to note that ``NaN($\bm{y}$)" indicates that the target accuracy could not be achieved within a limited number of communication rounds, where  { $\bm{y}$} represents the sum of energy costs from $1$ to $t$ communication rounds, in Tab.\ref{energy-table}.
The unit is watts (W) in Tab.\ref{energy-table}.
As shown in Tab.\ref{energy-table}, when the final communication round is reached, both LCSFLA and LCSFLA-bias consume the least amount of energy. This indicates that our design selects the most suitable clients for each round, and consequently, LCSFLA achieves optimal energy utilization. 
 {Specifically, in Case 1 of the MNIST dataset, LCSFLA requires only 63\% of the energy consumed by best baselines to achieve 95\% accuracy. In Case 2, LCSFLA requires only 39\% of the energy consumed by best baselines to achieve 95\% accuracy. For the Fashion-MNIST dataset, in Case 1, LCSFLA requires only 68\% of the energy consumed by best baselines to achieve 80\% accuracy.} In Case 2, LCSFLA requires only 32\% of the energy consumed by best baselines to achieve 80\% accuracy. 
For the CIFAR10 dataset, in Case 1, LCSFLA requires only 58\% of the energy consumed by best baselines to achieve 45\% accuracy. In Case 2, LCSFLA requires only 65\% of the energy consumed by best baselines to achieve 45\% accuracy. 
For the TSRD dataset, in Case 1, LCSFLA requires only 90\% of the energy consumed by best baselines to achieve 70\% accuracy.  {In Case 2, LCSFLA requires only 87\% of the energy consumed by best baselines, to achieve 70\% accuracy.}
These results further prove the energy efficiency advantage of our design and also show the advantage of client selection from the perspective of data balance in solving the Non-iid problem.

\subsection{Performance of Alleviating DCD}
\begin{figure}
\vspace{-0.5cm}
\centering
    \subfloat{\includegraphics[width=0.20\textwidth]{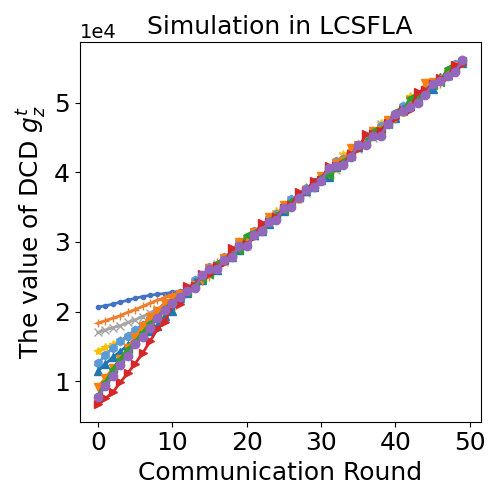}%
    \label{dcd1}}
\hfil
    \subfloat{\includegraphics[width=0.20\textwidth]{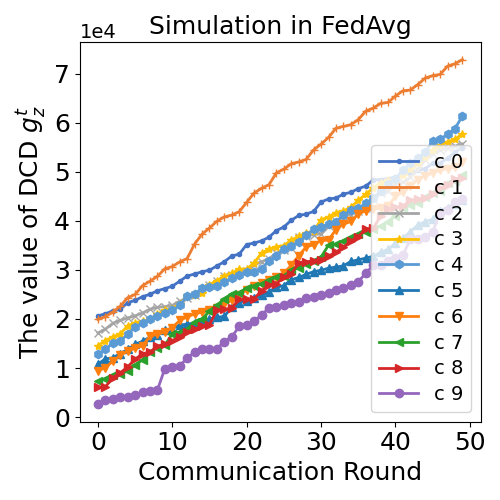}%
    \label{dcd2}}
\caption{The Change of DCD during training in LCSFLA and FedAvg.}
    \label{dcd}
    \vspace{-0.5cm}
\end{figure}

\begin{figure}[htp]
    \centering
    \includegraphics[width=0.45\textwidth]{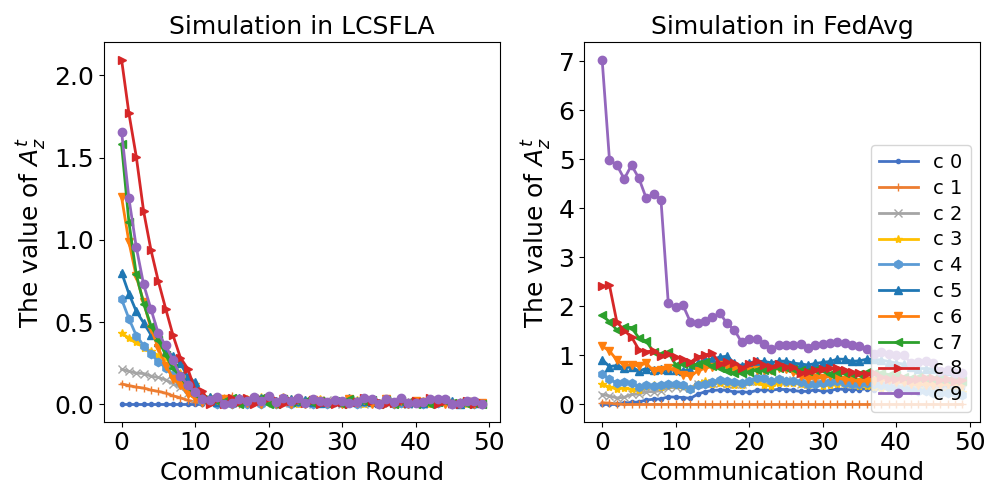}
    \caption{The Change of the average ratio of DCD $A^t$ during training in LCSFLA and FedAvg}
    \label{DCD2}
\end{figure}
To demonstrate the effectiveness of our mechanism in reducing the DCD gap across various categories, we set an initial imbalance at the initial communication round and compare the aggregation of different categories between LCSFLA and FedAvg on the MNIST dataset. Here are some explanations for Fig. \ref{dcd} and Fig. \ref{DCD2}, we can see that LCSFLA can effectively reduce the DCD, while FedAvg causes the model to train on a data distribution that deviates from the overall data distribution, resulting in a divergent outcome. 
The word ``category" is abbreviated to $c$ in the legend of Fig. \ref{dcd} and \ref{DCD2} for visual convenience.
From Fig. \ref{dcd}, the performance difference between the two algorithms on DCD can be seen more simply. The ratio of DCD to trained data size is used to reflect this difference. In Fig. \ref{dcd}, the difference ratio within data categories is calculated by the following formula,
\begin{equation}
\label{a}
  A_z^t = \frac{g_z^t}{o_z^t}.  
\end{equation}
The aggregation changes of data categories under different normalized factors $\sigma$ using LCSFLA are displayed in Fig. \ref{uw_all}. 
To visualize this change, we used the average ratio of DCD $A^t$ as an evaluation metric, in the form of
\begin{equation}
\label{aver-a}
    A^t = \frac{\sum_{z \in \textbf{Z}}A_z^t}{Z}.
\end{equation}
\begin{table*}[ht]
\caption{ {\textbf{Performance Evaluation Of Energy Efficiency}}}
\centering
\resizebox{\textwidth}{!}{
\begin{tabular}{cccccccccccccc}
\toprule
\multirow{2}{*}{Datasets}&\multirow{2}{*}{ToA@x} & \multicolumn{6}{c}{Case 1}&\multicolumn{6}{c}{Case 2} \\
\cmidrule(r){3-8}
\cmidrule(r){9-14}
&&LCSFLA&LCSFLA-bias&FedAvg&CSFedAvg&CAFL& {FL-DPS}&LCSFLA&LCSFLA-bias&FedAvg&CSFedAvg&CAFL& {FL-DPS} \\
\cmidrule(r){1-14}
\multirow{2}{*}{MNIST}&ToA@90&\textbf{58} &\textbf{49} &143 &175 &357 &84
                             &\textbf{13} &\textbf{15} &56  &114 &250 &39

\\
\cmidrule(r){2-14}
&ToA@95 &\textbf{130} &\textbf{146} &NaN(300) &346 &NaN(688)  &208
        &\textbf{49}  &\textbf{36}  &140      &352 &566       &129
\\

\midrule
\multirow{2}{*}{FASHION MNIST}&ToA@75   &\textbf{61} &\textbf{70} &98 &252 &411  &89
                                        &\textbf{7}  &\textbf{15} &32 &59  &120  &40
\\
\cmidrule(r){2-14}
&ToA@80 &\textbf{106}  &\textbf{127} &175 &333 &NaN(785) &155
        &\textbf{15}   &\textbf{28}  &47  &70  &201      &100
\\

\midrule
\multirow{2}{*}{CIFAR10}&ToA@45 &\textbf{85} &\textbf{55} &146 &282  &535 &NaN(507)
                                &\textbf{94} &\textbf{30}  &144 &196 &421 &240 
\\
\cmidrule(r){2-14}
&ToA@50 &\textbf{141} &\textbf{213} &200 &NaN(974)  &630  &NaN(507)
        &\textbf{107} &\textbf{50}  &240 &639       &741  &NaN(507)
\\

\midrule
\multirow{2}{*}{TSRD}&ToA@70 &\textbf{97} &\textbf{94}  &100 &826 &363      &107
                             &\textbf{68} &\textbf{44}  &96  &720 &NaN(396) &78
\\
\cmidrule(r){2-14}
&ToA@80  &\textbf{168} &\textbf{166}  &200 &NaN(884) &NaN(397) &NaN(212)
         &\textbf{161} &\textbf{115}  &196 &NaN(881) &NaN(396) &NaN(211)
\\

\bottomrule
\end{tabular}}
\label{energy-table}
\vspace{-0.25cm} 
\end{table*}

It is known that the difference degree of DCD gradually decreases as the value of $\sigma$ decreases, as shown in Fig. \ref{uw_all}. This indicates the normalized factor $\sigma$'s physical meaning that the extent and speed of resolving data imbalances. A smaller $\sigma$ represents a stronger willingness of the mechanism to choose MCs that can reduce data category differences. A larger $\sigma$ tends to select MCs with larger local data due to the limitation of energy consumption. The number of data category also influence the DCD. As shown in Fig. \ref{uw_all}, since the dataset only has three categories in the TSRD, achieving balance is less challenging, and therefore the client selection scheme using a larger $\sigma$ can accomplish the same effect compared to other datasets.
In addition, the influences of normalized factors $\sigma$ are different when the data distribution varies. The difference in the two cases shows that a lower $\sigma$ can reach the influences of effectively mitigating DCD, i.e. $A^t \leq 0.2$ when the data distribution isn't significant disparities. For example the effect of mitigating DCD when $\sigma=1$ in Case 1 achieves the the effect when $\sigma=2$ in Case 2. 

\begin{figure}[ht]
    \centering
    \includegraphics[width=0.45\textwidth]{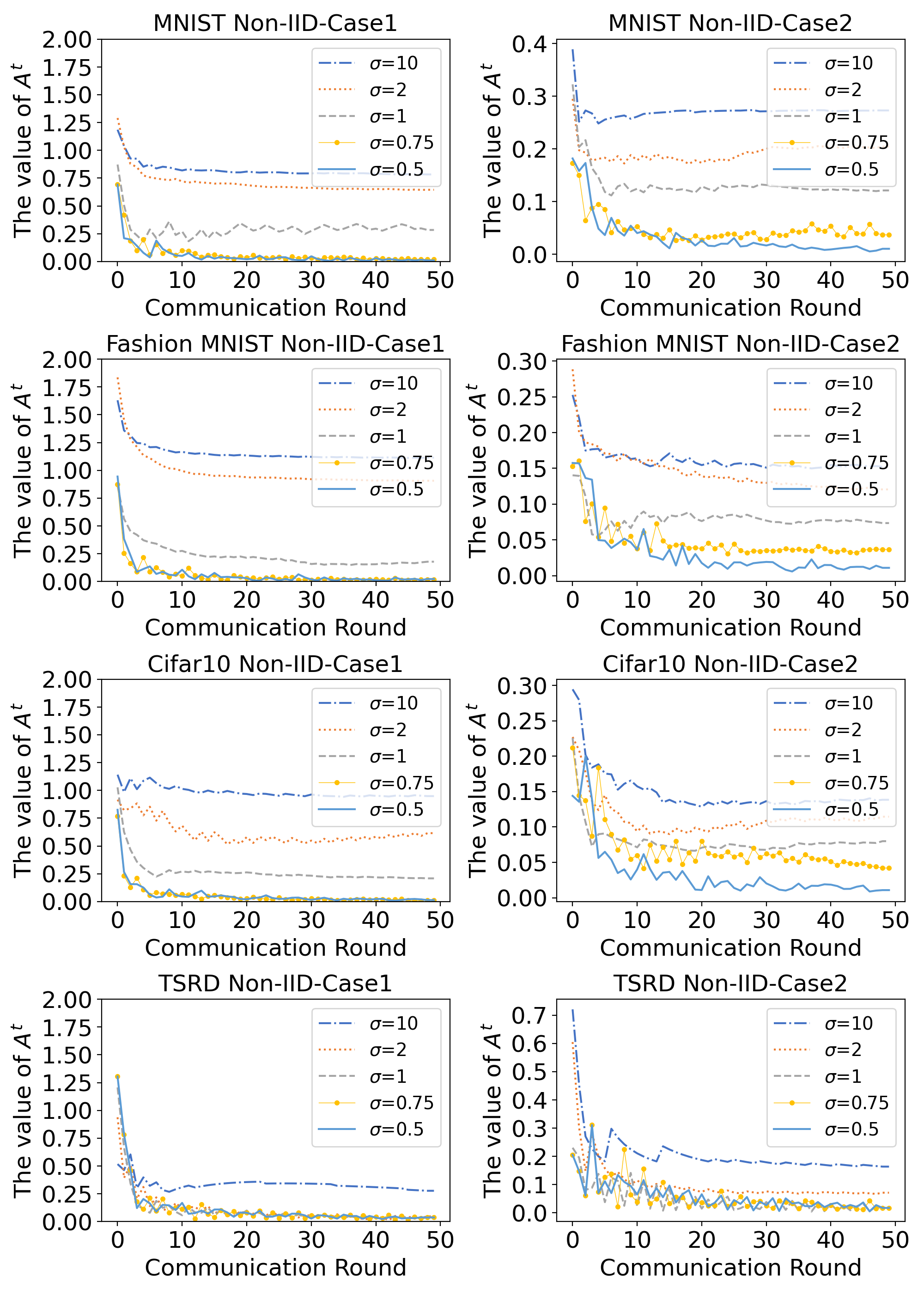}
    \caption{The Changes in DCD with different normalization factors}
    \label{uw_all}
\end{figure}
\subsection{Incentive Mechanism Analysis}
Our main objective is to show that the incentive mechanism we have designed is effective in motivating the CS to attract the MCs to participate in the FL. We then analyze five different algorithms from the perspective of social welfare. Figure \ref{sw} illustrates the changes in social welfare for LCSFLA, LCSFLA-bias, and three benchmark algorithms as the number of MCs varies from 20 to 100 in increments of 20. During the experiment, we did not deduct energy consumption from the estimation of the social welfare contribution of FedAuc in the first communication round. Therefore, we consider the average and overall social welfare from the second to the tenth communication round. The results indicate that our proposed approach can yield significantly higher social welfare than the benchmark algorithms.

\begin{figure}[ht]
\vspace{-0.25cm}
    \centering
    \includegraphics[width=0.45\textwidth]{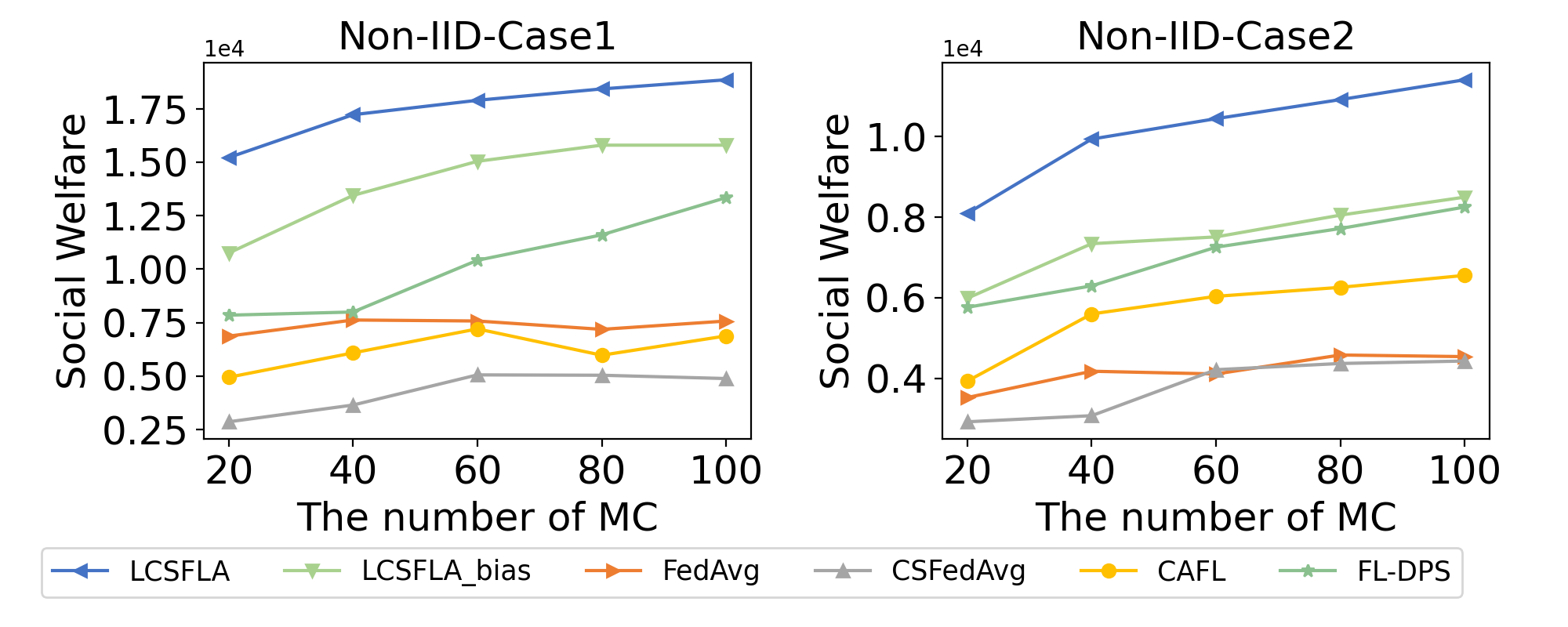}
    \vspace{-0.25cm}
    \caption{ {The changes of social welfare in different numbers of mobile clients}}
    \label{sw}
\end{figure}
 {In Case 1, when the number of MCs increases from 20 to 100, the social welfare of FedAvg, CSFedAvg, CAFL, and FL-DPS are about 39\%-45\%, 19\%-28\%, 32\%-40\%, and 46\%-71\% of the social welfare of LCSFLA, respectively. In Case 2, when the number of users increases from 20 to 100, the social welfare of FedAvg, CSFedAvg, CAFL, and FL-DPS is about 36\%-44\%, 31\%-40\%, 32\%-40\%, and 63\%-72\% of the social welfare of LCSFLA, respectively.} We can see that the social welfare of LCSFLA increases as the number of users increases, which is expected because LCSFLA will only choose some MCs that can make the system have greater social welfare when the number of clients increases. However, other algorithms have too much randomness, and the increase in the average social welfare of clients may be less than that of previous MCs when the number of MCs increases, resulting in a relatively tortuous increase in social welfare. In addition, these algorithms do not motivate CS to recruit more MCs to participate in FL, which is detrimental to the learning process of FL because the global model cannot learn more new data.

\section{Conclusion}
For the first time, we introduce a novel approach to address the challenge of data heterogeneity by selecting clients based on alleviating the discrepancy between data categories, namely LCSFLA. To achieve this objective, we propose a long-term data quality evaluation model that mainly combines the current training status that is the DCD, and the local data distribution of MCs to calculate their data quality. 
Subsequently, we design an auction mechanism to incentivize MCs and overcome information asymmetry in the IoV scenarios while ensuring relevant economic properties, based on the VCG mechanism.
Through the experiment of simulated data, we found that the LCSFLA can obtain more social benefits when the MC continues to join. This mechanism incentivizes the CS to attract more mobile clients to FL, which speeds up model training and final model performance. At the same time, we also found that compared with traditional algorithms such as FedAvg, LCSFLA can achieve the target with less energy consumption and communication rounds while maintaining the same model accuracy, showing good energy efficiency.
In summary, by designing a reasonable client selection method and incentive mechanism, LCSFLA can effectively solve the data heterogeneity and resource allocation problems in FL, thus improving the efficiency and performance of FL model training. This is of great significance to promote the widespread use of FL in the IoV, which is expected to bring new opportunities and challenges to the development of the field of smart vehicles.
\appendices
\def\thesection{\Alph{section}}%
\def\thesectiondis{\Alph{section}}%
\section{Proof of The Monotonic Increase of Data Preference Range Width with $\iota_{z}^t$}
\setcounter{equation}{0}
\renewcommand{\theequation}{A.\arabic{equation}}
In the following, we will prove that the width of the data preference range increases as $\iota_{z}^t$ increases. This requires proving that when $\iota_{z,\text{1}}^t \leq \iota_{z,\text{2}}^t$, the distance from the starting point $d_{z, m, \text{1}}^\text{l}$ and endpoint $d_{z, m, \text{1}}^\text{r}$ of the data preference range to $d_{z, m}=l=\iota_{z,\text{1}}^t$ is smaller compared to when $\iota_{z}^t = \iota_{z,\text{2}}^t$. When $u_{z, m}^{t}$ is equal to a value $\epsilon$, we can obtain the starting point and endpoint of the data preference range, as follows
\begin{equation}
    \begin{split}
        \epsilon = \alpha - \epsilon_\text{1} \left(\frac{\nu_z^t d_{z, m}-\iota_{ z}^t}{{\iota_{z}^t}}\right)^2,
    \end{split}
\end{equation}
where $\epsilon_\text{1} = \alpha(1-u^\text{c}_m)$. Due to the $\epsilon$ and $\alpha$ is the fixed value, we can obtain that $a^2 =\left(\alpha - \epsilon\right)/\epsilon_\text{1}$. Therefore, for a $\iota_{z, \text{1}}^t$, we can obtain that 
\begin{equation}
    \begin{split}
        \nu_{z, \text{1}}^{t, \text{l}} d_{z, m, \text{1}}^\text{l}/\iota_{z, \text{1}}^t = 1 - a, \\
        \nu_{z, \text{1}}^{t, \text{r}} d_{z, m, \text{1}}^\text{r}/\iota_{z,\text{1}}^t = 1 + a.
    \end{split}
    \label{sp1}
\end{equation}
For another $\iota_{z, \text{2}}^t$ and $\iota_{z, \text{2}}^t \geq \iota_{z, \text{1}}^t$, we can obtain that
\begin{equation}
    \begin{split}
        \nu_{z, \text{2}}^{t, \text{l}} d_{z, m, \text{2}}^\text{l}/\iota_{z, \text{2}}^t = 1 - a, \\
        \nu_{z, \text{2}}^{t, \text{r}} d_{z, m, \text{2}}^\text{r}/\iota_{z,\text{2}}^t = 1 + a.
    \end{split}
    \label{sp2}
\end{equation}
Suppose $d_{z, m}/\iota_{z}^t = k_{z, m}^{t}$, \eqref{sp1} and \eqref{sp2} can be reformulated as follows
\begin{equation}
    \begin{split}
        \exp(1- k_{z, \text{1}}^{t, \text{l}})k_{z, \text{1}}^{t, \text{l}} = 1 - a, \exp(1- k_{z, \text{2}}^{t, \text{l}})k_{z, \text{2}}^{t, \text{l}} = 1 - a, \\
        \exp(1- k_{z, \text{1}}^{t, \text{r}})k_{z, \text{1}}^{t, \text{r}} = 1 + a, \exp(1- k_{z, \text{2}}^{t, \text{r}})k_{z, \text{2}}^{t, \text{r}} = 1 + a.
    \end{split}
    \label{sp3}
\end{equation}
Due to $\exp{(1-x)}x$ is a monotone function when $x$ is positive value, we can obtain that
\begin{equation}
        k_{z, \text{1}}^{t, \text{l}} = k_{z, \text{2}}^{t, \text{l}},~k_{z, \text{1}}^{t, \text{r}} = k_{z, \text{2}}^{t, \text{r}}.
    \label{sp4}
\end{equation}
It is known that $k_{z, \text{1}}^{t, \text{l}} = k_{z, \text{2}}^{t, \text{l}} \leq 1$ and $k_{z, \text{1}}^{t, \text{r}} = k_{z, \text{2}}^{t, \text{r}} \geq 1$. So, we can obtain that
\begin{equation}
    \begin{split}
        d_{z, m, \text{2}}^\text{l} - d_{z, m, \text{1}}^\text{l} \leq \iota_{z,\text{2}}^t - \iota_{z,\text{1}}^t ,\\
        d_{z, m, \text{2}}^\text{r} - d_{z, m, \text{1}}^\text{r} \geq \iota_{z,\text{2}}^t - \iota_{z,\text{1}}^t .
    \end{split}
    \label{sp5}
\end{equation}
The \eqref{sp5} can be reformulated as follows
\begin{equation}
    \begin{split}
        \iota_{z,\text{2}}^t - d_{z, m, \text{2}}^\text{l} \geq \iota_{z,\text{1}}^t - d_{z, m, \text{1}}^\text{l}  ,\\
        d_{z, m, \text{2}}^\text{r} - \iota_{z,\text{2}}^t  \geq d_{z, m, \text{1}}^\text{r} - \iota_{z,\text{1}}^t .
    \end{split}
    \label{sp6}
\end{equation}
Based on the above, we can know that the distances from the starting point to $\iota_{z}^t$ and from the endpoint to $\iota_{z}^t$ increase as $\iota_{z}^t$ increases.

\section{Proof of The Individual Rationality (IR)}
\setcounter{equation}{0}
\renewcommand{\theequation}{B.\arabic{equation}}
The benefits of participating in FL need to be guaranteed to be non-negative for MC $m$.
\begin{align}
        \!\!\!\!\!\!U_m^t &= {q}_m^t {(r_m^t - E_m^t)}-\kappa_m^t
        \nonumber\\ 
        &= {q}_m^t {(r_m^t - E_m^t)}-\sum_{m\in \mathcal{M}_{-m}}{{\hat{q}}_{m}^{t}} \left( \hat{c}_m^t - \hat{E}_m^t \right)
        \nonumber\\
        &\qquad +\sum_{m\in \mathcal{M}_{-m}}{{q}_m^t} \left( c_m^t-    E_m^t\right)+{{q}_m^t}( c_m^t-  r_m^t)
        \nonumber\\
        &=\sum_{m\in \mathcal{M}_{-m}}{{q}_m^t} \left( c_m^t-    E_m^t\right)+ 
        {{q}_m^t}( c_m^t-   E_m)
        \nonumber\\
        &\qquad-\sum_{m\in \mathcal{M}_{-m}}{{\hat{q}}_{m}^{t}} \left( \hat{c}_m^t - \hat{E}_m^t \right)
        \nonumber\\
        &=\sum_{m\in \mathcal{M}}{{q}_m^t} ( c_m^t- E_m^t)-\!\!\sum_{m\in \mathcal{M}_{-m}}{{\hat{q}}_{m}^{t}} \left(  \hat{c}_m^t - \hat{E}_m^t \right).
\end{align}

From the previous description, ${{\hat{q}}_{m}^{t}}$ is a sub-optimal solution to maximization problem \eqref{WD01}, and we have
\begin{equation}
   \!\!\! U_m^t = \sum_{m\in \mathcal{M}}{{q}_m^t} ( c_m^t-  E_m) \geq \sum_{m\in \mathcal{M}_{-m}}{{\hat{q}}_{m}^{t}} \left( c_m^t -  E_m^t\right).
\end{equation}
Thus, $U_m^t \geq 0$.

\section{Proof of The Incentive Compatibility (IC)}
\setcounter{equation}{0}
\renewcommand{\theequation}{C.\arabic{equation}}
The auction mechanism ensures that the given benefit that the mobile client uploaded truthful bidding information is high bound of the benefit that the mobile client uploaded false bidding information. To this end, we need to ensure $U^t_m- \widetilde{U}^t_{m} \geq 0$, i.e.,
\begin{align}
    \!\!\!U_m^t& - \widetilde{U}_m^t\nonumber\\
        &= {q}_m^t (r_m^t - E_m^t) - \kappa_m^t - \widetilde{q}_m^t (\widetilde{r}_m^t - \widetilde{E}_m^t) + \widetilde{\kappa}_m^t \nonumber\\
        &= {q}_m^t (r_m^t - E_m^t) - \widetilde{q}_m^t (\widetilde{r}_m^t - \widetilde{E}_m^t) \nonumber\\
        &\quad + \sum_{m \in \mathcal{M}_{-m}} \hat{q}_m^t (\hat{c}_m^t - \hat{E}_m^t) - \sum_{m \in \mathcal{M}_{-m}} \widetilde{q}_m^t (\widetilde{r}_m^t - \widetilde{E}_m^t) \nonumber\\
        &\quad - \widetilde{q}_m^t (\widetilde{c}_m^t - \widetilde{r}_m^t) + {q}_m^t (c_m^t - r_m^t) \nonumber\\
        &\quad - \sum_{m \in \mathcal{M}_{-m}} \hat{q}_m^t (\hat{c}_m^t - \hat{E}_m^t) + \sum_{m \in \mathcal{M}_{-m}} {q}_m^t (c_m^t - E_m^t) \nonumber\\
        &= {q}_m^t (r_m^t - E_m^t) + \sum_{m \in \mathcal{M}_{-m}} {q}_m^t (c_m^t - E_m^t) \nonumber\\
        &\quad + {q}_m^t (c_m^t - r_m^t) - \widetilde{q}_m^t (\widetilde{r}_m^t - \widetilde{E}_m^t) \nonumber\\
        &\quad - \sum_{m \in \mathcal{M}_{-m}} \widetilde{q}_m^t (\widetilde{r}_m^t - \widetilde{E}_m^t) - \widetilde{q}_m^t (\widetilde{c}_m^t - \widetilde{r}_m^t) \nonumber\\
        &= \sum_{m \in \mathcal{M}_{-m}} {q}_m^t (c_m^t - E_m^t) + {q}_m^t (c_m^t - E_m^t) \nonumber\\
        &\quad - \sum_{m \in \mathcal{M}_{-m}} \widetilde{q}_m^t (\widetilde{r}_m^t - \widetilde{E}_m^t) - \widetilde{q}_m^t (\widetilde{r}_m^t - \widetilde{E}_m^t) \nonumber\\
        &= \sum_{m \in \mathcal{M}} {q}_m^t (c_m^t - E_m^t) - \sum_{m \in \mathcal{M}} \widetilde{q}_m^t (\widetilde{r}_m^t - \widetilde{E}_m^t).
\end{align}
It can be known that ${\widetilde{q}_{m}^t} $ is the optimal solution for when MC $m$ uploads untruthful bidding. However, ${\widetilde{q}_{m}^t}$ is the sub-optimal solution when the winner indicator changes and CS can obtain the truthful information. So, we can know that $U^t_m = \widetilde{U}^t_{m}$, when $ {{q}_m^t}={\widetilde{q}_{m}^t}$, and $U^t_m > \widetilde{U}^t_{m}$, when $ {{q}_m^t} \neq {\widetilde{q}_{m}^t}$. Based on the above-mentioned, we can conclude, i.e.,
\begin{equation}
    U^t_m \geq \widetilde{U}^t_{m}.
\end{equation}
\bibliography{reference.bib}

\begin{IEEEbiography}
[{\includegraphics[width=1in,height=1.25in,keepaspectratio]{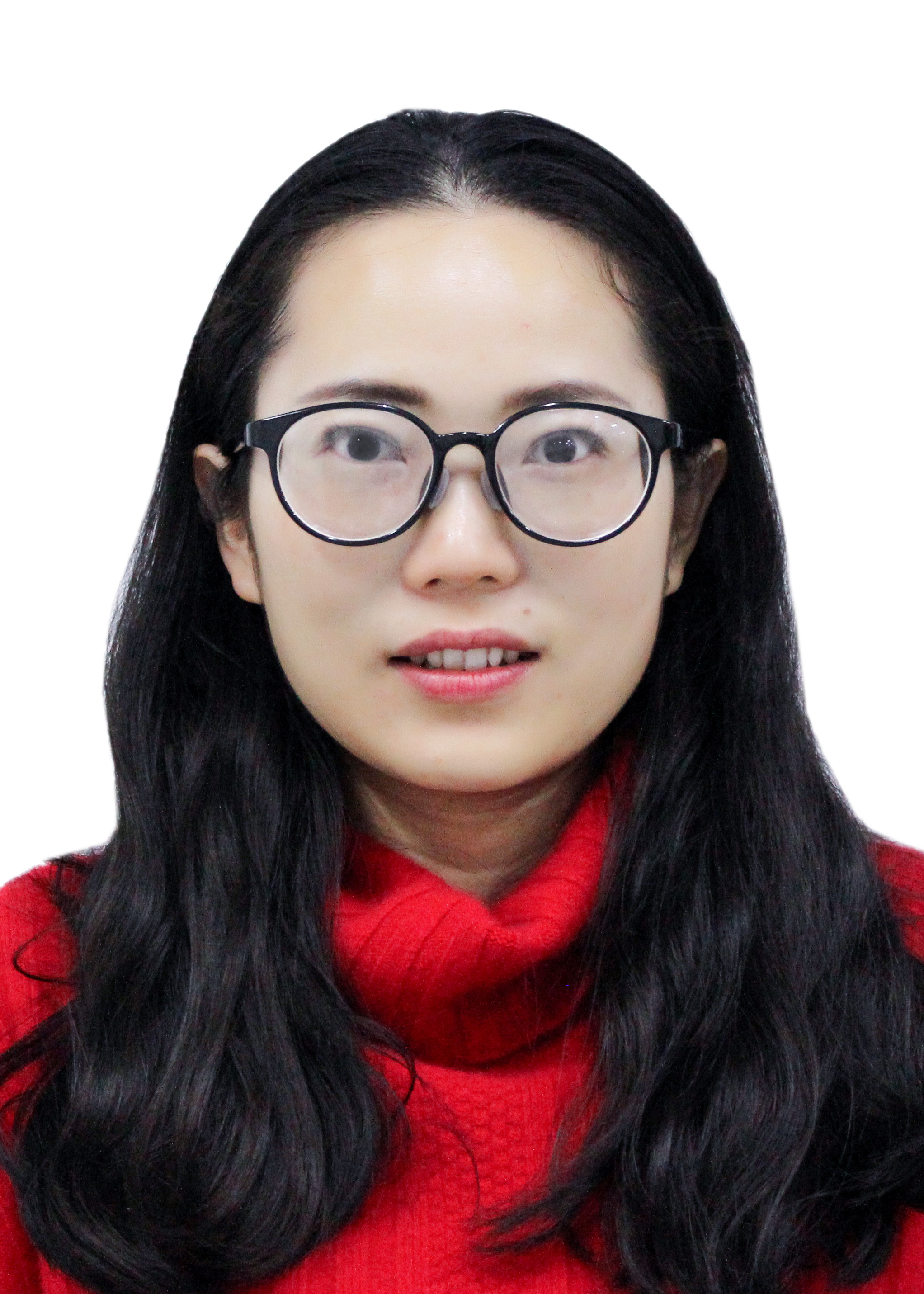}}]
{Jinghong Tan} received her B.E.\ degree in communication engineering from Shandong University, China, in 2014, and her Ph.D.\ degree in electronic engineering from the Singapore University of Technology and Design (SUTD), Singapore, in 2019. She is currently a lecturer at Yunnan University, China.
Her current research interests include cloud/edge radio access networks, edge computing, and machine learning.
\end{IEEEbiography}

\begin{IEEEbiography}
[{\includegraphics[width=1in,height=1.25in, clip, keepaspectratio]{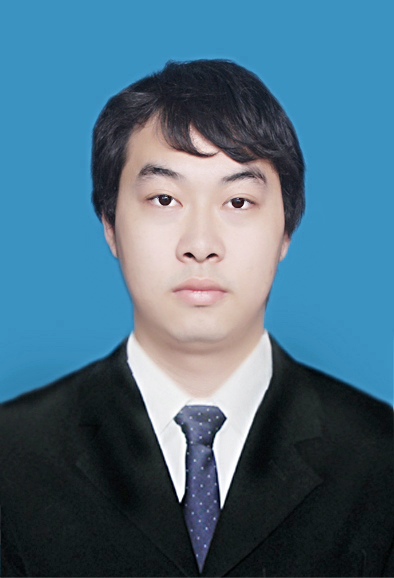}}]{Zhian Liu }(Student Member, IEEE) is currently pursuing the M. Eng. degree in software engineer from Yunnan University, Kunming, China. His research interests include federated learning and incentive mechanisms.
\end{IEEEbiography}


\begin{IEEEbiography}[{\includegraphics[width=1in,height=1.25in,keepaspectratio]{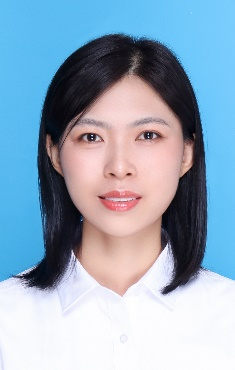}}]
{Kun Guo} (Member, IEEE) received the B.E. degree in Telecommunications Engineering from Xidian University, Xi'an, China, in 2012, where she received the Ph.D. degree in communication and information systems in 2019. From 2019 to 2021, she was a Post-Doctoral Research Fellow with the Singapore University of Technology and Design (SUTD), Singapore. Currently, she is a Research Professor with the School of Communications and Electronics Engineering at East China Normal University, Shanghai, China. Her research interests include wireless edge computing and intelligence, as well as non-terrestrial networks.
\end{IEEEbiography}

\begin{IEEEbiography}
[{\includegraphics[width=1in,height=1.25in,clip,keepaspectratio]{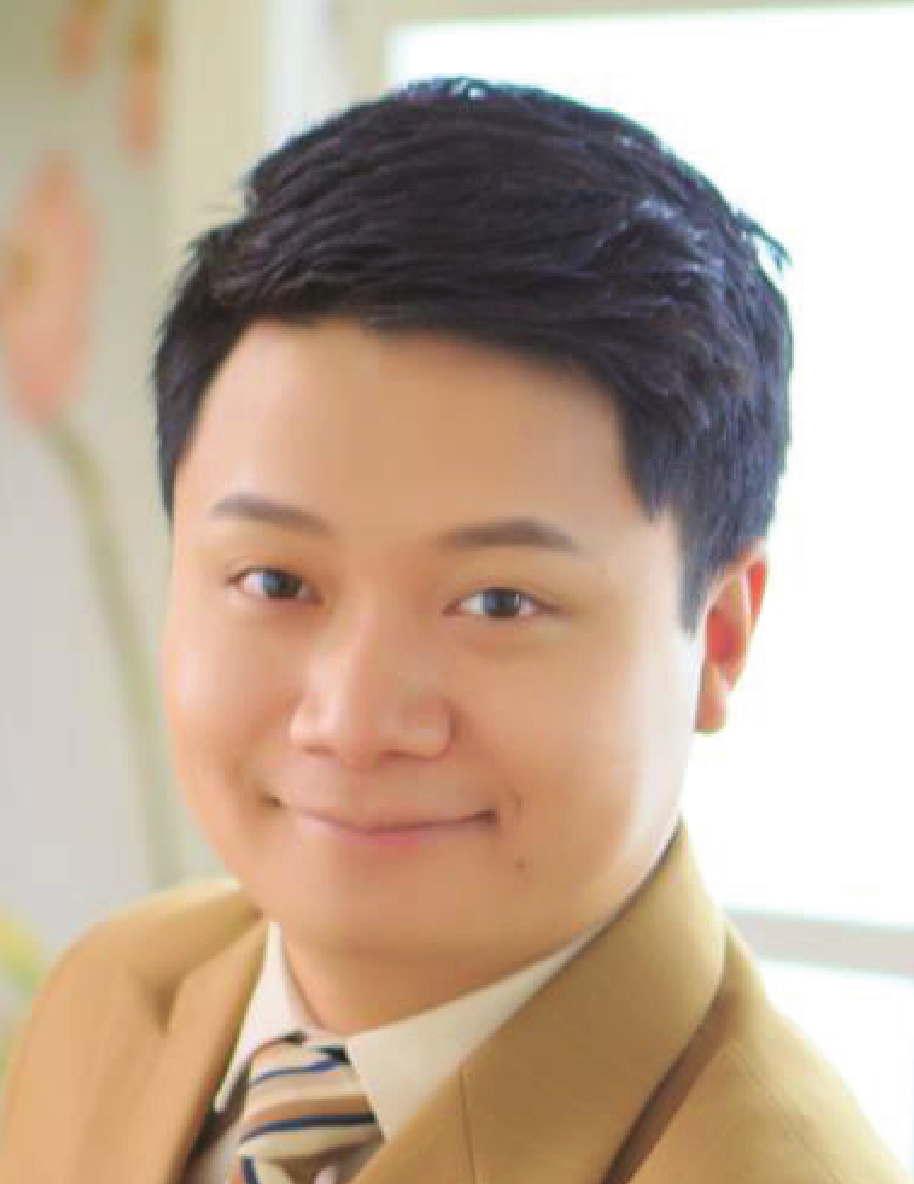}}]
{Mingxiong Zhao} (S'15-M'17) received the B.S. degree in Electrical Engineering and the Ph.D. degree in Information and Communication Engineering from South China University of Technology (SCUT), Guangzhou, China, in 2011 and 2016, respectively. He was a visiting Ph.D. student at University of Minnesota (UMN), Twin Cities, MN, USA, from 2012 to 2013 and Singapore University of Technology and Design (SUTD), Singapore, from 2015 to 2016, respectively. Currently, Dr. Zhao is a Full Professor and the Donglu Young Scholar with Yunnan University (YNU), Kunming, China. He also serves as the Director of the Cybersecurity Department with the National Pilot School of Software, and has been an Outstanding Young Talent of Yunnan Province since 2019. His current research interests include network security, mobile edge computing, and edge AI techniques. 

Dr. Zhao is currently serving as a Youth Editor for the {\scshape Journal of Information and Intelligence}, and a committee member of the Technical Committee on Data Security of the China Communications Society.
\end{IEEEbiography}
\end{document}